\documentclass[11pt]{article}

\usepackage[a4paper, margin=1in]{geometry}
\usepackage{lmodern}
\usepackage{graphicx}
\usepackage{authblk}

\usepackage[T1]{fontenc}

\usepackage{hyperref}
\usepackage{microtype}
\usepackage{amsmath}
\providecommand{\citep}{\cite}
\providecommand{\citet}{\cite}

\usepackage{amssymb}
\usepackage{mathtools}
\usepackage{amsthm}

\usepackage{enumitem}

\usepackage{dsfont}
\usepackage{comment}

\usepackage{cleveref}

\usepackage[most]{tcolorbox}

\renewenvironment{proof}[1][Proof]{%
  \begin{tcolorbox}[
    enhanced,
    breakable,
    colback=gray!5,
    colframe=gray!40,
    boxrule=0.4pt,
    arc=2pt,
    left=6pt,
    right=6pt,
    top=4pt,
    bottom=4pt,
    before skip=8pt,
    after skip=8pt,
  ]
  \noindent\textbf{#1. }\ignorespaces
}{%
  \hfill $\square$
  \end{tcolorbox}
}

\newcommand{\R}{\mathds{R}}
\newcommand{\N}{{\mathds{N}}}
\newcommand{\tr}{\operatorname{tr}}
\newcommand{\scal}[2]{{\langle{#1},{#2}\rangle}}
\newcommand{\X}{\mathcal{X}}
\DeclareMathOperator{\Tr}{Tr}
\DeclareMathOperator{\KL}{KL}
\DeclareMathOperator*{\argmin}{arg\,min}

\theoremstyle{plain}
\newtheorem{theorem}{Theorem}[section]
\newtheorem{lemma}[theorem]{Lemma}
\newtheorem{corollary}[theorem]{Corollary}

\AddToHook{env/lemma/begin}{\crefalias{theorem}{lemma}}
\AddToHook{env/corollary/begin}{\crefalias{theorem}{corollary}}

\AddToHook{cmd/appendix/after}{\crefalias{section}{appendix}\crefalias{subsection}{subappendix}}

\title{Distributionally Faithful Imputation via Positive Semi-Definite Kernel Density Estimation}
\author[1]{Andrea Basteri\thanks{Corresponding author.}}
\author[2]{Carlo Ciliberto}
\author[3]{Alessandro Rudi}
\affil[1]{Inria Paris, ENS--PSL}
\affil[2]{University College London, United Kingdom}
\affil[3]{Bocconi University, Italy}
\affil[ ]{\texttt{andrea.basteri@inria.fr}, \texttt{c.ciliberto@ucl.ac.uk}, \texttt{alessandro.rudi@sdabocconi.it}}
\date{}

\begin{document}
\maketitle
\begin{abstract}
Missing values undermine statistical inference and machine–learning pipelines, yet most imputation methods rely on heuristics or restrictive parametric assumptions that ignore the joint data distribution. We recast imputation under missing completely at random (MCAR) as density estimation from masked observations: estimate a distribution whose observed marginals exactly match those in the data. Leveraging positive semi–definite (PSD) kernel densities we obtain a convex empirical–risk problem with closed–form marginals, solvable by a Newton interior–point method. We further introduce a non-convex extension that supports adaptive hyperparameter selection through alternating minimization. The resulting PSD-Impute model yields both single and multiple imputations from the same fitted density, enjoys statistical consistency, beating the curse of dimensionality for very regular probabilities. Experiments indicate competitive distributional accuracy compared with popular imputation baselines, suggesting strong practical promise.
\end{abstract}

\tableofcontents

\section{Introduction}
\label{sec:introduction}

Missing data appears everywhere in modern data science. An electronic health-record may lack key laboratory values because instruments failed; MRI scans often omit contrasts for cost reasons; ecology studies record only a subset of species interactions; even a movie--rating matrix is intentionally sparse so that users see only a fraction of items. Each of these scenarios leaves practitioners with the same quandary: proceed with a biased and potentially invalid analysis, or \emph{impute} the gaps before any downstream task. Over four decades of research have produced a zoo of techniques, from the classical Expectation--Maximisation (EM) algorithm \citep{dempster1977maximum} to sophisticated deep generative models \citep{yoon2018gain}, yet today’s pipelines still debate which strategy to trust.

A central point of contention is \emph{what} constitutes a ``good'' imputation. Popular root--mean--square error (RMSE) figures reward methods that predict the conditional mean of each missing entry. Such pointwise targets ignore the multivariate dependencies that motivate imputation in the first place and typically understate uncertainty \cite{naf2023imputation}. Recent empirical work therefore advocates distributional metrics—energy distance \cite{szekely2017energy} or optimal transport (OT) \cite{villani2009optimal}—which measure how closely an imputed data set reproduces the \emph{joint} structure of the ground truth \cite{muzellec2020missing, naf:hal-04521894}. Striving for distributional fidelity, however, brings new modelling challenges: it requires a coherent density, not merely conditional means, and pushes optimisation beyond the safe havens of convexity.

This paper fills that gap by reframing imputation under the classical \emph{missing completely at random} (MCAR) mechanism as a \emph{marginal--matching} density--estimation problem.
Given a data--generating density $p^*$, we estimate a distribution that \emph{exactly} matches the observed marginals, and under an identifiability condition it allows one to reconstruct the original distribution $p^*.$ 
To make the idea practical, we introduce a non-convex heuristic that can be optimized through alternating minimization. The procedure alternates between updating the model parameters by solving a convex problem with fixed hyperparameters and adaptively selecting the model hyperparameters. The non--convex formulation can be seen as a generalisation of the convex one, when the hyperparameters are fixed. The parameter update exploits \emph{positive semi-definite} (PSD) kernel densities \cite{marteau2020non,rudi2021psd}, whose closed-form marginals and convex log-likelihood enable an efficient structured Newton interior-point solver.

\paragraph{Contributions.} (1) We cast MCAR imputation as convex marginal Kullback--Leibler minimization, bridging kernel density estimation and missing--data methodology; (2) we introduce PSD kernel densities to imputation and develop an efficient interior--point algorithm and a non--convex generalization for adaptive hyperparameter selection, (3) we establish statistical consistency with respect to the masked KL divergence, with an excess-risk bound of $\mathcal{O}(N^{-\frac{s}{6s+2d}})$, where $s$ denotes the smoothness of the target distribution $p^*$ and $d$ is the ambient dimension; (4) our single model unifies single and multiple imputation via conditional means or sampling, (5) preliminary experiments show that \textsc{PSD-Impute} achieves competitive performance against strong baselines, motivating further engineering work.

\section{Related Work}
\label{sec:related-work}

Research on imputation has evolved along several intertwined strands. \emph{Joint-model} approaches were the first to emerge; by specifying a full likelihood and fitting it via EM or Bayesian MCMC they offer principled uncertainty quantification, yet their reliance on multivariate Normality and cubic scaling curtail widespread adoption. \emph{Fully Conditional Specification} (FCS) methods, epitomised by MICE \citep{van2011mice}, circumvent rigid distributional assumptions by letting users choose a regression model for each variable. This flexibility comes at the price of $\mathcal{O}(d^{2})$ regressions per feature \citep{josse2024consistency} and, when the chosen conditionals are incompatible, the lack of a coherent joint density that complicates distribution-level evaluation \citep{Murray2018}.

Parallel to these developments, \emph{low-rank} and matrix-completion techniques treat missingness as a signal-reconstruction problem. Convex nuclear-norm programs such as SoftImpute \citep{hastie2015matrix} achieve impressive performance on recommender data but assume low intrinsic rank and optimise pointwise criteria. Tree-based ensemblers like missForest \citep{stekhoven2012missforest} have become popular for their ease of use and ability to capture nonlinear interactions, yet evidence shows they can severely distort multivariate structure under distributional metrics \citep{naf:hal-04521894}. The deep-learning surge brought methods such as GAIN \citep{yoon2018gain}, VAEM \citep{ma2020vaem}, and diffusion-based CSDI \citep{tashiro2021csdi}. While powerful in principle, these models optimise non-convex objectives aligned with RMSE, exhibit high variability across runs, and can under-represent uncertainty \citep{wang2021deep}.

A separate line views imputation itself as a \emph{distributional} problem. Muzellec et al.\ \citep{muzellec2020missing} propose an optimal-transport barycentric formulation that directly targets distributional distance, but the resulting optimisation is non-convex and suffers from the curse of dimensionality. Work by Le Morvan and Varoquaux\ \citet{morvan2024imputation} and theoretical analyses by Josse et al.\ \citep{josse2024consistency} underscore that the quality of imputation must be judged in the context of the entire analysis pipeline. These insights motivate our own density-matching perspective.

Against this backdrop, \textsc{PSD-Impute} distinguishes itself by uniting four desirable properties: (i) hyperparameter selection via cross-validation and adaptation during training, coupled with a \emph{convex} training objective for fixed hyperparameters; (ii) \emph{distributional fidelity} of the observed marginals, ensured via a marginal matching problem; (iii) the \emph{universality} of kernel densities, free of Gaussian or low-rank assumptions; and (iv) a vectorisable implementation, due to the structured Newton steps. 
Moreover, PSD models can be efficiently marginalized and conditioned~\citep{rudi2021psd}, and sampling is possible~\citep{marteau2022sampling}.

\section{Problem Formulation}\label{sec:problem}\label{sec:methodology}

\paragraph{Notation.}  Fix the dimension $d\in\mathbb{N}$ and denote $[d]=\{1,\dots,d\}$.  For a (possibly ordered) subset $S\subseteq[d]$ let $NS=[d]\setminus S$.  Given $x=(x_1,\dots,x_d)\in\mathbb{R}^d$ we write $x_S=(x_j)_{j\in S}\in\mathbb{R}^{|S|}$.  For any probability density $p$ on $(\mathbb{R}^d,\mathrm{d}x)$ we denote by $p_S$ the \emph{marginal} of $p$ over $S$,
\[
  p_S(v_S)\;=\;\int_{\mathbb{R}^{|NS|}} p(v_S,z_{NS})\,\mathrm{d}z_{NS }.
\]
As a running illustration, when $d{=}5$ and $S{=}\{1,2,4\}$ we have $NS=\{3,5\}$ and $p_S(v_1,v_2,v_4)=\int p(v_1,v_2,z_3,v_4,z_5)\,\mathrm{d}z_3\mathrm{d}z_5$.

\paragraph{Missing--data mechanism.}  We adopt the classical taxonomy of \cite{little2019statistical}.  A mask random variable $S\sim\mu$ is \emph{Missing Completely At Random (MCAR)} when it is independent of the data $X\sim p^{\star}$.  This article focuses on the MCAR regime, leaving MAR/MNAR extensions for future work. Under MCAR the joint density factorises as
\[
  (X,S)\sim p^{\star}(x)\otimes\,\mu(S).%
  \]

\paragraph{Learning goal.}  Given $N$ i.i.d. observations $(x_{i},S_i)$ drawn from the above process our aim is to reconstruct the \emph{full} data–generating density $p^{\star}$.  The information supplied by the data limits what can be recovered: at best we can demand
\begin{equation}\label{eq:marginal_condition}
  p_S \;=\; p^{\star}_S,\quad \text{for $\mu$--almost--every } S\subseteq[d].
\end{equation}

Nevertheless, equality of marginals is \emph{necessary} and can be turned into a tractable objective.  Define the \emph{population risk}
\begin{equation}\label{eq:ideal-risk}
  \mathcal{L}(p)\;:=\;\sum_{S\subseteq[d]} \mu(S)\,\KL\bigl(p^{\star}_S\,\|\,p_S\bigr),
\end{equation}
where $\KL$ is the Kullback–Leibler divergence. %
The choice of KL is motivated by: %
(i) its information–theoretic interpretation; (ii) its convexity in the space of probability distributions, which can facilitate optimisation.  
The next lemma relates \eqref{eq:marginal_condition} and \eqref{eq:ideal-risk}.

\begin{lemma}\label{lem:projection}
$\mathcal{L}(p)=0\;$ if and only if $p$ satisfies \eqref{eq:marginal_condition}.
\end{lemma}
\begin{proof}
The ``if'' part follows from the fact that $\KL(a\,\|\,a)=0.$ The converse follows from the positivity of the $\KL$ and the fact that $\KL(a\,\|\,b)=0$ implies $a=b$ almost everywhere.
\end{proof}

\paragraph{Empirical risk minimisation.}  Because $p^{\star}$ is unknown we replace \eqref{eq:ideal-risk} with its unbiased estimator
\begin{equation}\label{eq:empirical-risk}
  \widehat{\mathcal{L}}(p)\;:=\; -\frac{1}{N}\sum_{i=1}^{N} \log p_{S_i}(x_{i,S_i}) \;\;(+\text{const}).
\end{equation}
This is precisely the negative log–likelihood of the masked observations.  Minimising \eqref{eq:empirical-risk} therefore follows the standard statistical–learning recipe: pick a rich model class, add appropriate regularisation, and solve the resulting optimisation problem.

\paragraph{Identifiability.} A core problem in performing a proper imputation is the correct identifiability of the data--generating distribution $p^*.$ That is, given two d--dimensional distributions $p$ and $q$, we ask the question ``given a subset of marginals $\mathcal{S}\subseteq 2^{[d]},$ is it true that $p_S=q_S$ for every $S\in\mathcal{S}$ implies that $p=q$?''. We observe that if $[d]\in \mathcal{S}$ the answer is clearly yes. In the non--trivial case where $[d]\notin \mathcal{S},$ if $p$ and $q$ are general distributions the answer is negative. For example, for $d=2$, we just need to take $p$ and $q$ to be two Gaussians with the same marginals and different covariance, and equality of the marginals does not imply equality of the joint distributions. The answer is positive if we add structure to the shape of $p$ and $q$, for example if all the components are independent for both $p$ and $q,$ then $\mathcal{S}=\{\{i\}:i=1,\dots,d  \}$ is enough. 
Our algorithm ensures that we can find a joint distribution $p_Q$ that correctly approximates the marginals, and in principle the learned joint distribution might be different. The question whether the learned model bounds in some metric the underlying data distribution under extra assumptions is non--trivial and it is deferred to future work. In the following to get the $joint$--$marginals$ bound we make the assumption that the probability of seeing a full observation is strictly positive.

\section{Algorithm}\label{sec:algorithm}

\paragraph{PSD kernel density model.}  Fix a set of \emph{anchor points} $\mathcal{W}=\{w_1,\dots,w_\ell\}\subset\mathbb{R}^d$ that act as centres for the kernel basis.  Importantly, $\mathcal{W}$ is an \emph{external mesh}: the $w_j$ need not coincide with (possibly incomplete) training examples but can be chosen on a regular grid, drawn i.i.d.
from a simple proposal (e.g.
uniform on a bounding box) or learned by a separate procedure.  This decouples the representation from the missing--value mechanism.  Section~\ref{sec:comp} details the particular heuristic we employ in our experiments. Together with a Gaussian kernel $k_\eta(x,x')=\exp\bigl(-\eta\|x-x'\|^2\bigr)$ we define the feature map.  The associated feature map is $\phi(x)=(k_\eta(x,w_1),\dots,k_\eta(x,w_\ell))\in\mathbb{R}^{\ell}$.  For any PSD matrix $Q\in\mathcal{S}^{\ell}_{+}$ we obtain a non–negative function
\[p_Q(x)=\phi(x)^\top Q\phi(x).\]
Normalising via the moment matrix $H:=\int \phi(x)\phi(x)^\top\mathrm{d}x=c_{2\eta}(k_{\eta/2}(w_i,w_j))_{i,j=1,\dots,\ell}$ 
(closed form for Gaussian kernels) and the constraint $\Tr(QH)=1$ yields a valid density.  The resulting \emph{PSD kernel density family}, introduced  by~\cite{marteau2020non,rudi2021psd}, enjoys four key properties {\em at the same time}, compared to alternative approaches \cite{rudi2021psd}, motivating our choice:
\begin{enumerate}[label=(\roman*)]
  \item \textbf{Universality.}  As $\ell\to\infty$ and anchors grow dense, the span $\{p_Q:Q\succeq0\}$ is dense in $L_1$, see \cite{rudi2021psd}; we can approximate any continuous density.
  \item \label{eq:closed_form_marginals} \textbf{Closed–form marginals.}  For every subset $S$ one can marginalise analytically:
  \[
    p_{Q,S}(x_S)=\,\tr\!\bigl(Q\, (K_S(x_S)\circ H_{NS})\bigr),
  \]
  where $K_S(x_S)=\phi_S(x_S)\phi_S(x_S)^ T,$ with $\phi_S(x_S)=(k_{\eta_S}(x_S,w_{1,S}),\dots,k_{\eta_S}(x_S,w_{\ell,S}))$ is a rank–one kernel matrix and 
  $H_{NS}=\int_{(-1,1)^{|NS|}}\phi_{NS}(x_{NS})\phi_{NS}(x_{NS})dx_{NS}$
  depends only on $NS$ and the anchor points $\mathcal{W}$. %
  For the Gaussian kernel the matrix $H_{NS}$ is explicitly computable, see \cite{rudi2021psd}.
  \item \textbf{Convex likelihood.}  After normalisation, $Q\mapsto-\log p_Q(x)$ is convex; the ERM objective is therefore convex \emph{in the parameters}.  %
  \item \textbf{Optimal approximation of smooth densities.}  For targets that are sums of squares of Sobolev functions of order $s$, the least--squares PSD estimator of \cite{rudi2021psd} attains the $\mathcal{O}(N^{-s/(2s+d)})$ $L^2$--rate with only $\ell\approx N^{d/(2s+d)}$ anchors.
\end{enumerate}

\paragraph{Empirical risk over PSD models.}
\begin{theorem}[Convex ERM for PSD densities]\label{thm:erm-psd}
Let $\{(x_i,S_i)\}_{i=1}^N$ be incomplete observations with masks $S_i\subset[d]$.  
Define
\[
  A_i := K_{S_i}(x_{i,S_i}) \circ H_{NS_i},
\]
where $K_{S_i}(x_{i,S_i})$ and $H_{NS_i}$ are defined in \ref{eq:closed_form_marginals}.
Then the empirical risk minimiser within the PSD-density family $\{p_Q:Q\succeq0,\,\Tr(QH)=1\}$ is obtained by solving
\[
  \widehat Q_N = \argmin_{Q\succeq0,\,\Tr(QH)=1}\;\bigl[-\tfrac1N\sum_{i=1}^N \log\Tr(QA_i)\bigr].
\]
Moreover, the map $Q\mapsto-\log\Tr(QA_i)$ is convex, hence the optimisation is convex and global optima are reachable with deterministic solvers.
\end{theorem}
\begin{proof}[Sketch]
Because $A_i\succeq0$, the function $Q\mapsto-\log\Tr(QA_i)$ is convex; the feasible set is convex as well.  The objective equals the negative log-likelihood of the observed marginals, making the solution the empirical-risk minimiser.  %
\end{proof}

\paragraph{Regularized maximum likelihood.}  To minimise the empirical risk \eqref{eq:empirical-risk} we solve the regularized risk
\begin{tcolorbox}[title=Convex objective under MCAR]
\vspace{-1.5ex}
\begin{equation*}\tag{P--C}\label{eq:primal}
  \argmin_{Q\succeq0}\; f(Q)
  := -\frac1N\sum_{i=1}^N
      \log(\operatorname{Tr}(QA_i)+\alpha_i)
  +\lambda\operatorname{Tr}(QA_0)
      \qquad\text{s.t. }\operatorname{Tr}(QH)=1.
\end{equation*}
\vspace{-1.5ex}
\end{tcolorbox}
The matrices $A_i$ and $H$ are pre–computable once anchors are fixed, as we see in Theorem \ref{thm:erm-psd}. The trace term controls the RKHS norm (preventing over–fitting). The $\alpha_i-$term makes the loss function upper bounded.

\paragraph{The shortcoming of the convex objective.} The convex objective enjoys provable strong theoretical properties, but the weak point is the choice of the $precision$ $\eta$ and the $anchor$ $points$ $\mathcal{W}.$ In the proofs it is mathematically tempting to assume that the anchors are $i.i.d.$ uniformly distributed, but in practice this would make the problem intractable for a large number of anchors. To mitigate this issue, we generalize the convex formulation~\eqref{eq:primal} to the following non--convex objective

\begin{tcolorbox}[title=Non--convex objective under MCAR]
\vspace{-1.5ex}
\begin{equation*}\tag{P--NC}\label{eq:primal_non_convex}
\begin{aligned}
  \argmin_{\substack{
    Q\succeq0\\
    |\mathcal{W}|=\ell\\
    \eta\in\R^d_{++}
  }}\; f(Q,\mathcal{W},\eta)
  := {}& -\frac1N\sum_{i=1}^N\log\left(\Tr(QA_i)+\alpha_i\right)
  +\lambda\Tr(Q A_0)\\
  & \text{s.t. }\Tr(QH)=1.
\end{aligned}
\end{equation*}
\vspace{-1.5ex}
\end{tcolorbox}
where this time the matrices $A_i$ and $A_0$ and $H$ are also optimized with respect to the anchor nodes (given a pre--selected size $\ell$), and precision. With this formulation, the problem becomes significantly dependent on the starting point, but it might require less computation power. We discuss some possible heuristics to find good initial conditions that yield effective solutions. Observe that, when the anchor points $\mathcal{W}$ and the precision $\eta$ are fixed, the non-convex Problem~\eqref{eq:primal_non_convex} reduces to Problem~\eqref{eq:primal}.

\section{Theoretical Guarantees}\label{sec:theory}

Missing--completely--at--random (MCAR) settings impose a stark information constraint: for any incomplete record we observe only the coordinates selected by a random mask $S\sim\mu$.  Consequently the \emph{only} aspect of the true density $p^{\star}$ accessible to us is the family of observable marginals $\{p^{\star}_S\}_{S\sim\mu}$.  A principled estimator should therefore aim to \emph{exactly} reproduce those marginals and measure fidelity through the masked--Kullback--Leibler risk $\mathcal L(p)$ defined in~Equation (\ref{eq:ideal-risk})---the strongest loss we can hope to minimise under MCAR.  %

We prove that the estimator coming from the optimization of Problem~\eqref{eq:primal}, i.e. $\hat p=p_{\hat Q_N}$ learns the observable marginals with a non--asymptotic rate of convergence.  Full proofs appear in Appendix~\ref{appx:theory}. We defer the study of the estimator coming from the non--convex loss, namely Problem~\eqref{eq:primal_non_convex}, to future work.

Having introduced the estimator, we now establish when it is provably reliable.

\bigskip

\noindent\textbf{Proof roadmap.} We first define the masked risk and state two mild assumptions, then bound separately the \emph{learning} and \emph{approximation} errors before stitching them into a joint-KL consistency theorem.

\subsection{Setup and risk}
Let $p^{\star}$ be the data--generating density and $\mu$ the distribution of masks $S\subset[d]$.  For any density $p$ define the masked Kullback--Leibler risk as in Equation (\ref{eq:ideal-risk}). The empirical minimiser $\hat p$ arises from~\eqref{eq:primal}.

\paragraph{Assumptions.}\label{ass:smooth-ident}
(i) (Smoothness) $p^{\star}$ is a sum of squares of Sobolev functions $\in W^s_2(\R^d)\cap L^{\infty}(\R^d)$ with $s > d/2$, and $p|_D$ is a probability measure, with $D=(-1,1)^d$. 
For example, this holds if $p^\star$ is bounded away from zero, or $p^*$ is differentiable up to the boundary and the set of zeros is finite with a positive definite Hessian at each zero. \\
(ii) (Mesh-points) The points of the mesh $\cal W$ are chosen uniformly $i.i.d.$ on the $d$--dimensional hypercube $D$.

Assumption (ii) is needed to get the approximation properties of the models \cite{rudi2021psd}. Assumption (i) comes directly from \cite[Proposition 5]{rudi2021psd} (and generalized in \cite{marteau2024second}) and guarantees that $p^\star$ can be efficiently approximated by PSD models.
 
\subsection{Risk decomposition}\label{sec:risk_dec}
Fix $\lambda, \nu > 0$. Denote by $\hat{p}$ the probability distribution that is the solution of our algorithm, i.e. the PSD model defined as
\begin{equation}\label{eq:phat}
    \hat{p}(x) := (1-\nu) p_ {\hat{Q}}(x) ~+~ \nu u(x), 
\end{equation}
where $\hat{Q}$ is the solution of Problem~\eqref{eq:primal}, for the given $\lambda, \eta, \nu$, and $u(x) = 1(x)/V_D$ is the uniform probability on $D = (-1,1)^d$ with $V_D = \int_D dx$. Moreover, let $\tilde{p}$ be the reference best PSD model in the population case, i.e. 
$$\tilde{p}(x) = (1-\nu) p_{\tilde{Q}}(x) + \nu u(x),$$
where $\tilde{Q} \in \R^{\ell \times \ell}$ satisfies
\[
  \tilde{Q} = \argmin_{\substack{
    Q\succeq0\\
    \Tr(QH)=1
  }}\; \tilde{\cal L}(Q)
  := {\cal L}((1-\nu)p_Q + \nu u)
  + \lambda \operatorname{tr}(Q A_0).
\]
We split the excess risk of $\hat p$ into a statistical and an approximation component:
\begin{equation*}
  \tilde{\cal L}(\hat{Q} ) ~~=~~ \underbrace{\bigl[
  \tilde{\cal L}(\hat{Q}) -\tilde{\cal L}(\tilde{Q})\bigr]}_{\text{learning error}}+
  \underbrace{\tilde{\cal L}(\tilde Q)}_{\text{approximation error}}.
\end{equation*}
The first term is the price of estimating $\hat{Q}$ from $N$ samples; the second measures how well the PSD class on the mesh points ${\cal W}$ can approximate $p^{\star}$. Observe that for theoretical reasons we have multiplied the PSD predictor by $(1-\nu),$ and this is not present in the definition of Problem~\eqref{eq:primal}. This is not an issue because if we consider $\alpha_i=\nu\int_{D_{NS_i}} 1 dx_{NS_i}/(V_D(1-\nu)),$ and we optimize Problem~\eqref{eq:primal}, and then we consider $\alpha'_i=\nu\int_{D_{NS_i}} 1 dx_{NS_i}/V_D,=\alpha_i(1-\nu)$ with $D_{NS_i}=(-1,1)^{|NS_i|}$, and we optimize Problem~\eqref{eq:primal} with $(1-\nu)Tr(QA_i) + \alpha'_i$ as the argument of the logarithm, then the optimization procedure stays the same. This is equivalent to adding the constant $\log(1/(1-\nu))$ to the loss function, which does not perturb the extraction of the $\argmin$. Moreover, when the argument is $(1-\nu)Tr(QA_i) + \alpha'_i,$ this is equivalent to marginalizing Equation (\ref{eq:phat}), and plugging it into the loss function, so that we can leverage the Rademacher theory to get complexity bounds.

\subsection{A good candidate}
The theory of PSD models \cite{rudi2021psd} tells us under assumption (i) it is possible to efficiently approximate a probability with a small mesh $\cal W$, i.e. that there exists a good candidate $\bar{Q} \in \R^{\ell \times \ell}, \bar{Q} \succeq 0$, such that $p_{\bar{Q}}$ approximates $p^*$ very effectively and with a small Frobenius norm for $\bar{Q}$. The following theorem makes this statement precise.
\begin{theorem}\label{thm:exists-barQ}
Let $p^*$ satisfy assumption (i), and let $\cal W$ satisfy assumption (ii). Let $\epsilon \in (0, \epsilon_0]$ with $\epsilon_0$ be a constant dependent on $d$, and let $\delta \in (0,1)$ be a number.
When
$$
\ell \geq C_1  \epsilon^{-d/s} (\log(C_1/\epsilon\delta))^{d/2+1}, \quad \eta = C_2 \epsilon^{-2/s}/\log(2/\epsilon),
$$
then there exists $\bar{Q} \in \R^{\ell \times \ell}$ and $\bar{Q} \succeq 0$, such that the following holds with probability $1-\delta$
$$
\|p_{\bar{Q}} - p^*\|_{L^2(D)} \leq \epsilon, \qquad \operatorname{tr}(\bar{Q} K) \leq C_3\ell,\quad \tr(\bar{Q}H)=1.
$$
Here $H$ and $K$ are suitable kernel matrices, and the constants $C_1, C_2, C_3$ depend only on $s, d, \|p^*\|_{W^s_2(D)}$.
\end{theorem}

See Appendix \ref{appx:proof-exists_barQ} for the full proof. Given this existence theorem, we now proceed to analyze the approximation error.

\subsection{Approximation error}
The existence of $\bar{Q}$ allows us to bound $\tilde{Q}$, since $\bar{Q}$ belongs to the same hypothesis space and by construction $\tilde{\cal L}(\tilde{Q}) \leq \tilde{\cal L}(\bar{Q})$. We use moreover the fact that, by the information monotonicity property of KL with respect to marginalization (see e.g. Thm 2.5.3 in \cite{cover1999elements}), we have that
$$
{\cal L}(p) = \int KL(p^*_S | p_S) d\mu(S) \leq KL(p^*|p),
$$
for any probability density $p$ on $D$,
and the upper bound of KL with total variation due to the reverse Pinsker inequality, since $p^*$ is bounded from above and $(1-\nu)p_{\bar{Q}} + \nu u$ is bounded from below. The following theorem combines these ideas.

\begin{theorem}\label{thm:approx-error-tildeQ}
Under the assumptions of \cref{thm:exists-barQ} on $p^*$, ${\cal W}$ and under the same choice dependent on $\epsilon,\delta$ for $\ell, \eta$, the following holds with probability $1-\delta$
$$
\tilde{\cal L}(\tilde{Q}) ~~\leq~~ C_1\log(C_2/\nu)(\epsilon + \nu) ~+~  C_3 \ell \lambda.
$$
\end{theorem}
See Appendix \ref{appx:proof-approx-error-tildeQ} for the full proof. Having tackled the approximation error, we are now ready to address the learning error.
\subsection{Learning error}

The following Theorem states the main generalization bound. The first key ingredient is a bound on the trace of the ERM solution $\hat Q$ in terms of the optimal solution in the population case, i.e. $\tilde Q.$ The second key ingredient is a non-trivial bound on the hypothesis class of the problem.
\begin{theorem}[Generalization bound]\label{thm:learn}
Under Assumption (i) in ~\ref{ass:smooth-ident}, for any $\lambda > 0, \nu \in (0, \varepsilon_0], \delta \in (0,1)$ and $\ell, \eta$ chosen as in \cref{thm:approx-error-tildeQ} with $\epsilon = \nu$, we have 
\[
  \tilde{\cal L}(\hat Q) - \tilde{\cal L}(\tilde Q)
  \le \frac{C}{\nu}
    \left(\ell + \lambda + \frac{1}{\lambda}\right)
    \log^{3/2}\left(
      \frac{\ell\log(1/(\lambda\nu))}{\lambda\nu}
    \right)
    \sqrt{\frac{\log(16/\delta)}{N}},
\]
with $C$ depending on $s,d,\|p^*\|_{W^s_2(D)}$ only.
\end{theorem}
See Appendix \ref{appx:proof_th_learn} for the full proof and Appendix \ref{appx:lemmas} for additional lemmas.

\subsection{Combined convergence}
Putting the two pieces together, balancing Theorem~\ref{thm:approx-error-tildeQ} and Theorem~\ref{thm:learn} with $\ell\asymp N^{d/(6s+2d)}$ yields the following theorem.

\begin{theorem}[KL consistency of marginals]\label{thm:marginal}
Let $\delta \in (0,1)$. Let $p^*$ satisfy assumption (i), and the mesh points ${\cal W}$ satisfy (ii). Let $\hat{p}$ be the solution of Problem~\eqref{eq:primal}, defined as Equation (\ref{eq:phat}). By choosing 
\[
  \nu = \tilde O\left(N^{-\frac{s}{6s+2d}}\right), ~~
  \lambda = \tilde O\left(N^{-\frac{s+d}{6s+2d}}\right), ~~
  \ell = O\left(
    N^{\frac{d}{6s+2d}}
    \left(\log\left(\frac{N}{\delta}\right)\right)^{d/2+1}
  \right), ~~
  \eta = O\left(\ell^{2/d}\right),
\]
the following holds with probability at least $1-\delta$
\[
  \int \! \mathrm{KL}(p^\star_S\,\|\,\hat{p}_S)\,\mathrm d\mu(S)
  = O\bigl(N^{-s/(6s+2d)}\bigr).
\]

\end{theorem}
\begin{proof}
Note that, by definition of $\tilde{\cal L}$, we have
$$
\int \! \mathrm{KL}(p^\star_S\,\|\,\hat{p}_S)\,\mathrm d\mu(S) \leq \left(\tilde{\cal L}(\hat{Q}) - \tilde{\cal L}(\tilde{Q})\right) + \tilde{\cal L}(\tilde{Q}).
$$
By bounding the first term via the generalization error bound in \cref{thm:learn} with $\epsilon = \nu$, and the second by the approximation result in \cref{thm:approx-error-tildeQ} we obtain the final result with probability at least $1-\delta$. More details on the algebraic computations are given in Appendix \ref{appx:detail_th_bound_the_marginal}.
\end{proof}

\paragraph{Identifiability.} Given a bound on the masked KL, we would like to derive a corresponding bound on the KL divergence between the target distribution and the learned predictor. A simple condition that allows one to derive such a bound is assuming that the probability of seeing a full observation is strictly positive, i.e. we have that $\mu([d])>0.$
\begin{theorem} Under the assumption that the probability of observing a full observation is positive, i.e. that $\mu([d])>0,$ we get the bound 
\begin{align*}
    \mathrm{KL}(p^\star\,\|\,\hat{p}\,)\, =O\bigl(N^{-s/(6s+2d)}\bigr), 
\end{align*}    
\end{theorem}
\begin{proof}
The proof is just the observation that
\begin{align*}
    \mathrm{KL}(p^\star\,\|\,\hat{p}\,)\,
    \leq\frac{1}{\mu([d])}\int \! \mathrm{KL}(p^\star_S\,\|\,\hat{p}_S)\,\mathrm d\mu(S)
    =O\bigl(N^{-s/(6s+2d)}\bigr),
\end{align*}
as we wanted.
\end{proof}
Of course, the above rate depends on the constant $\mu([d])$ that might be exponentially small in high dimensions.
It would be interesting if it were possible to get a bound under more structured assumptions on the shape of the target distribution. For example, if the target factorizes into independent components and no full observations are available, is the bound on the masked KL enough to derive a bound on the full KL?
In this simple example, we are essentially asking if the learned predictor, which we know to be close to the target in the $L^2$--norm, has low mutual information. It is not clear if we can derive this property from Problem~\eqref{eq:primal} or Problem~\eqref{eq:primal_non_convex}, and we defer this research direction to future work.

\section{Computation and Optimisation}\label{sec:comp}

\paragraph{Matrix construction.}\label{subsec:matrix}
Let $\mathcal W=\{w_1,\dots,w_\ell\}\subset\R^d$ be the anchor mesh and denote by $k_\eta(x,x')=\exp(-\eta\|x-x'\|^2)$ the isotropic Gaussian kernel. 
The $moment$ $matrix$ is defined as
\begin{equation*}
  H_{NS,\alpha\beta}=\int_{\R^{|NS|} } k_\eta(x_{NS},w_{NS,\alpha})\,k_\eta(x_{NS},w_{NS,\beta})\,d x,
\end{equation*}
In practice for computational reasons we integrate over the whole domain $\R^d$, because the integrals can be computed explicitly, see Section \ref{app:psd_conditional_mean_imputation_psd} and Appendix F of \cite{rudi2021psd}.

For an incomplete sample $(x_i,S_i)$ with observed coordinates $S_i$ and complement $N S_i$ we split the kernel bandwidth as $\eta=(\eta_{S_i},\eta_{NS_i})$ and form the rank--one matrix $K_{S_i}(x_{i},S_i)=\phi_{S_i}(x_{i,S_i})\phi_{S_i}(x_{i,S_i})^\top$.  Element--wise the data–dependent matrix reads
\begin{equation}\label{eq:A}
  A_{(x_i,S_i),\alpha\beta}=A_{i,\alpha\beta}=K_{S_i}(x_{i,S})_{\alpha,\beta}H_{NS_i,\alpha\beta}
\end{equation}
Both $H$ and $\{A_i\}$ are \emph{pre--computable} once the anchors and bandwidths are fixed, and their sparse--Kronecker structure is key to efficient optimisation. In the non--convex version of the algorithm, these matrices must be recomputed every time.

\paragraph{Regularised objective.}\label{subsec:objective}
Given $A_0:=K_\eta$, $\{A_i\}_{i=1}^N$ and $H$ which are $\mathcal{W}$ and $\eta$ dependent we solve the trace and log--det--regularised empirical risk
\begin{equation*}\tag{P--NC--R}\label{eq:primal-short}
\begin{aligned}
  \min_{\substack{
    Q\succeq0\\
    \mathcal{W}:|\mathcal{W}|=m\\
    \eta\in\R^d_{++}
  }}\; f(Q,\mathcal{W},\eta)
  := {}& -\frac1N\sum_{i=1}^N
    \log(\Tr(QA_i)+\alpha_i)
  +\lambda\Tr(QA_0)
  -\mu\log\det Q\\
  & \text{s.t. }\Tr(QH)=1.
\end{aligned}
\end{equation*}
where $\lambda$ controls RKHS complexity while $\mu>0$ controls the effect of the log--det term, which acts both as a self--concordant barrier and, if interpreted from a Bayesian viewpoint, as a Wishart prior density~\cite{olkin1964multivariate}.
The equality constraint simply normalises $p_Q$ into a density. Solving Problem~\eqref{eq:primal-short} is \textbf{PSD-Impute}, and its optimization is carried out by alternating minimization over the matrix $Q$ and the pair $(\mathcal{W},\eta)$. If we fix the pair $(\mathcal{W},\eta)$ and we solve only for $Q$ we get \textbf{PSD--Convex}, so that the convex problem is a special instance of the non--convex problem. We show the pseudocode in Section \ref{alg:PSD_Impute}.

\paragraph{Interior--point Newton solver.}\label{subsec:newton}
The gradient $g$ and Hessian $G$ of the barrier--loss are fully derived in Appendix~\ref{appx:comput}. For the optimization of the convex Problem~\eqref{eq:primal}, we used a damped Newton method algorithm with a $-\mu\log\det(Q)$ barrier to guarantee the positiveness of the matrix $Q$. %
Explicit details about the algorithm are recalled in \cref{eq:interior-damped-Newton}.

\paragraph{Hyper-parameter optimization.}\label{subsec:hyper_param_optimization}
The matrices $A_i's$ depend implicitly on the anchor points $\cal W$ and the precision $\eta.$
The starting point of the anchors is fundamental since a good position can dramatically simplify the optimization process. To avoid the selection of too large a number of anchor nodes, we draw from multiple heuristics. First, contrary to the theory in Section \ref{sec:theory}, we draw the anchors from a subset of an imputed copy of the dataset and not uniformly $i.i.d.$ from the hypercube. A cross--validation procedure chooses the best among a set of candidate anchor nodes. The score of a set of candidates is measured by the RMSE between the imputed values of the entries covered by an increased mask and the original ground truth of these extra missing values. More details are given in Section \ref{sec:experiments}.

\paragraph{Why do we need the non--convex formulation?} In Section \ref{sec:theory} and Appendix \ref{appx:theory} we prove statistical guarantees for the learned predictor of our model. A natural question then is, why do we need to introduce the non--convex loss, and construct heuristics to choose the anchor points? %
When jointly optimizing over $(Q,(W,\eta))$,  we are optimizing a non-convex loss and it is difficult to characterize mathematically how this affects our theoretical guarantees.
In a worst-case scenario (for example when the density is highly complex), learning the anchor nodes may not provide a significant advantage over selecting them uniformly at random in the hypercube.
However, in practice distributions often exhibit structure, and learning the anchor nodes can allow the method to identify meaningful locations without requiring them to densely cover the entire domain. 

In the limiting case where there are no missing data, if observed points are used as anchor points, their locations are naturally influenced by the regions where the underlying distribution places more mass. Therefore, they may provide a better representation of the data distribution than uniformly sampled anchor points. Our intuition is that using an imputed version of the data as anchor points, rather than selecting anchor points $i.i.d.$, should lead to better positioned anchors. In particular, this approach may allow the method to capture meaningful structure without requiring the anchor points to densely cover the entire domain.

\begin{figure}[t]
  \centering
  \includegraphics[width=\linewidth]{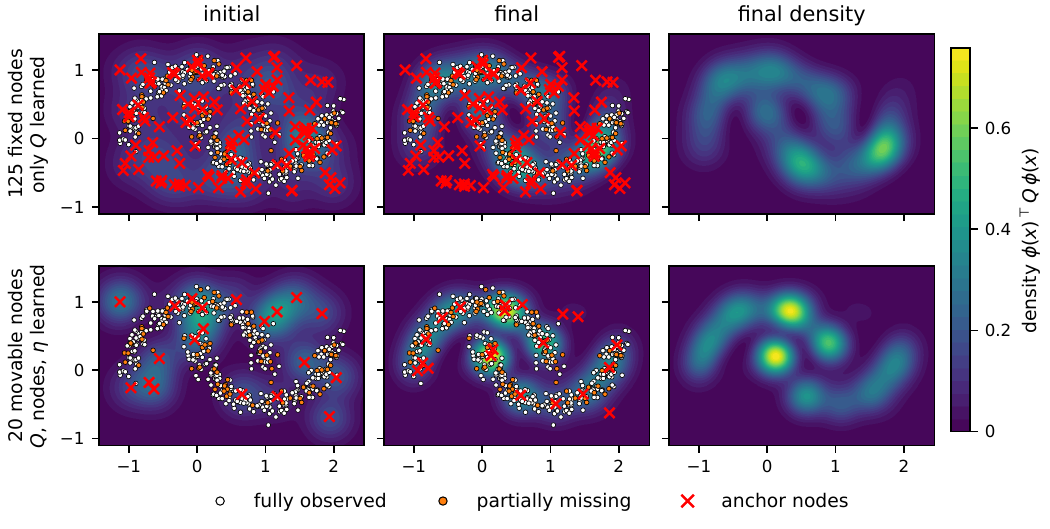}%
  \caption{Qualitative results on 2--D toy datasets. Top: Matrix $Q$ learned. Bottom: Matrix $Q$ and hyperparameters $(\mathcal{W},\eta)$ learned. More details in Figure \ref{fig:toy}.}
  \label{fig:toy-main}
\end{figure}

\begin{figure}[t]
\centering
\includegraphics[width=1\linewidth]{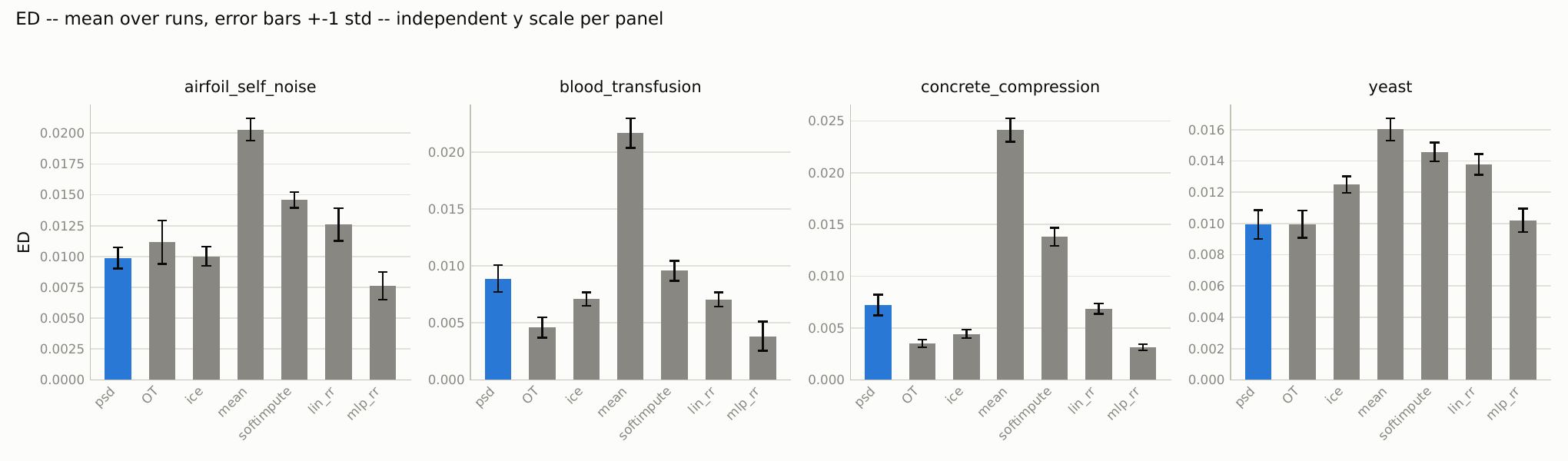}
\includegraphics[width=1\linewidth]{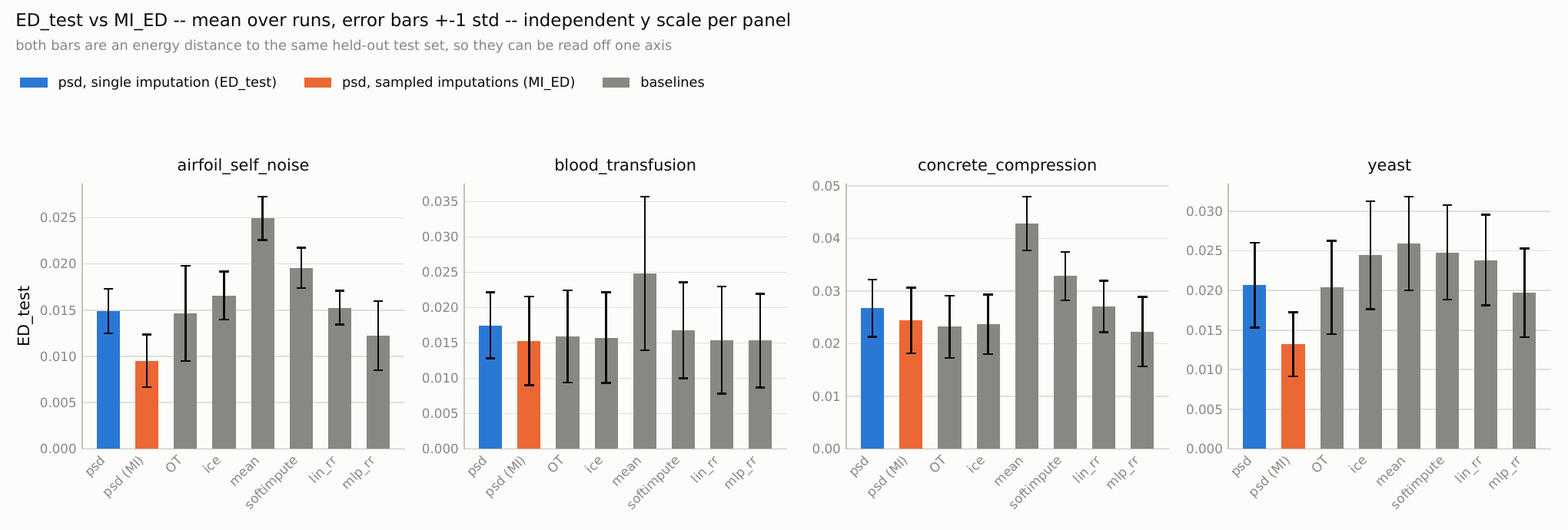}
\caption{\textbf{Distributional accuracy on real datasets.}  Top: imputation on the training set compared with the ground truth. Bottom: the imputed dataset is compared with a held-out test set. The orange bar corresponds to drawing $10$ samples per imputed point from the conditional PSD model as a sanity check. More information in Figure \ref{fig:dist_metrics_ed}.}
\label{fig:dist_metrics_ed-main}
\end{figure}

\paragraph{What can we do with the learned estimator?} 
Identifiability currently holds under the assumption that the probability of observing a full row is positive, i.e. $\mu([d])>0$. Without this assumption, we learn just a density whose marginals match the observed marginals, so that our problem is a marginal matching density estimation problem, and the fitted density can be used to draw samples from a plug--in conditional distribution. However, a fitted density can give plausible-looking samples while missing dependence that the observation pattern cannot identify.  
Obtaining guarantees of identifiability under structural assumptions on $p^\star$ is a possible research direction.

\subsection{Complexity vs.\ statistical accuracy}\label{subsec:complexity}

Here we examine the tradeoff between computational complexity and statistical accuracy for the convex algorithm, i.e. Problem~\eqref{eq:primal}, as we have developed the theory within this framework.

To gauge tractability we derive the sample and anchor budgets required for a target KL error $\varepsilon$. To attain a Kullback–Leibler error $\varepsilon$ we follow the prescription of Theorem~\ref{thm:marginal}:
\[
  \nu = \tilde O\left(N^{-\frac{s}{6s+2d}}\right), ~~
  \lambda = \tilde O\left(N^{-\frac{s+d}{6s+2d}}\right), ~~
  \ell = O\left(
    N^{\frac{d}{6s+2d}}
    \left(\log\left(\frac{N}{\delta}\right)\right)^{d/2+1}
  \right), ~~
  \eta = O\left(\ell^{2/d}\right).
\]%
The rate of convergence of Theorem \ref{thm:marginal} is $O(N^{-\frac{s}{6s+2d}}).$ So, if the threshold is of order $\varepsilon=O(N^{-\frac{s}{6s+2d}}),$ we get that we need $N=O(\varepsilon^{-6-\frac{2d}{s}})$ observations, with the regularizer value $\lambda=O(\varepsilon^{1+d/s} )$
and $\ell=O(\varepsilon^{-\frac{d}{s}} \log(1/\varepsilon)^{d/2+1})$ nodes.

\paragraph{Per–iteration cost.}  With $\ell$ anchors, one Newton step costs $\mathcal O(N\ell^{3}+N^2\ell^2+N^{3})$ flops and requires $\mathcal O(N^2+N\ell^{2})$ in memory, using the SMW formula, see Section \ref{sct:direct_method}. Building and storing the Hessian explicitly costs $O(N\ell^4)$ in memory and time, and solving the linear system costs $O(\ell^6)$ in time. 

\paragraph{Number of iterations.}  The self-concordance of $-\log\det$~\cite[Exercise 3.3.3, Section 4.4]{nesterov1994interior} bounds the number of Newton steps by $\mathcal O\bigl(\sqrt{\ell}\,\log(1/\varepsilon)\bigr)$.

\paragraph{Total complexity.}  Substituting the above values of $N$ and $\ell$ gives, for the SMW method

\[
\underbrace{\mathcal O\!
\bigl(
\varepsilon^{-18-\frac{6d}{s}}
\,(\log\tfrac1\varepsilon)^{d/2+1}\bigr)}_{\text{time}}
\quad\text{and}\quad
\underbrace{\mathcal O\!\bigl(\varepsilon^{-12-\frac{4d}{s}
}\bigr)}_{\text{memory}}.
\]

See Appendix \ref{sct:direct_method} for more details. Working with the Hessian will cost

\[
\underbrace{\mathcal O\!\bigl(
\left(
\varepsilon^{-6-\frac{6d}{s}}
\,
\right)(\log\tfrac1\varepsilon)^{d/2+1}\bigr)}_{\text{time}}
\quad\text{and}\quad
\underbrace{\mathcal O\!\bigl(
\varepsilon^{-6-\frac{6d}{s}
}\bigr)}_{\text{memory}}.
\]

We defer a more comprehensive experimental comparison to future work focused on the engineering aspects of the algorithm.

\section{Discussion}
\label{sec:discussion}

\paragraph{Limitations and future directions.} Our work has both practical and theoretical limitations that suggest several future research directions. 
First, faithful imputation under new missing patterns requires recovering the original target. Our current approach to the theoretical identifiability supposes access to a sample of fully observed data, and future work aims to relax this hypothesis. Second, the min--max optimality of our result remains unexplored, as the asymptotic rate we get as $s\to +\infty$ is slower than the classical $O(N^{-1/2})$ rate of convergence. Third, our theory holds for the convex Problem (\ref{eq:primal}) in the fixed anchors--precision setting for uniformly $i.i.d.$ anchors, and we aim to study the more challenging non--convex formulation, i.e. Problem (\ref{eq:primal_non_convex}) in future work. Fourth, our algorithm involves numerous hyperparameters that can be difficult to tune, and since we are optimizing over a space of matrices both memory and time complexity might easily become intractable. Future work could develop heuristics for efficient hyperparameter selection, and could mitigate the computational cost by leveraging, for example, sketching techniques or inducing-point methods to reduce the computational costs.

\paragraph{Extension.} We believe that this work can be extended to get excess risk bound also with PSD models defined on all $\R^d$, and not restricted to the hypercube $D.$ The idea would be first to extend Theorem D.4 in \cite{rudi2021psd} to get Morrey's bound, and then using the property that $k_{\eta_{2}}\geq k_{\eta}$ induces an RKHS inclusion that can be used to get the Rademacher bound. We leave this direction for future work.

\section{Conclusion}
\label{sec:conclusion}

We reframed imputation as convex marginal KL minimisation, introduced PSD kernel densities with closed‐form marginals, and solved the resulting optimisation with an efficient interior‐point Newton method.  The estimator enjoys consistency guarantees and yields both single and multiple imputations. Our experiments yield promising results.

\bibliographystyle{alpha}
\bibliography{bib}

\newcommand{\etalchar}[1]{$^{#1}$}
\begin{thebibliography}{NSMM23}

\bibitem[Bha97]{bhatia2013matrix}
Rajendra Bhatia.
\newblock {\em Matrix Analysis}, volume 169 of {\em Graduate Texts in
  Mathematics}.
\newblock Springer, New York, 1997.

\bibitem[Bil98]{bilmes1998gentle}
Jeff~A. Bilmes.
\newblock A gentle tutorial of the {EM} algorithm and its application to
  parameter estimation for {Gaussian} mixture and hidden {Markov} models.
\newblock Technical Report TR-97-021, International Computer Science Institute,
  Berkeley, CA, 1998.

\bibitem[Bin19]{binette2019note}
Olivier Binette.
\newblock A note on reverse {Pinsker} inequalities.
\newblock {\em IEEE Transactions on Information Theory}, 65(7):4094--4096,
  2019.

\bibitem[Bre11]{brezis2011functional}
Haim Brezis.
\newblock {\em Functional Analysis, Sobolev Spaces and Partial Differential
  Equations}.
\newblock Springer, 2011.

\bibitem[BV04]{boyd2004convex}
Stephen Boyd and Lieven Vandenberghe.
\newblock {\em Convex Optimization}.
\newblock Cambridge University Press, 2004.

\bibitem[Byr13]{byrne2013alternating}
Charles~L Byrne.
\newblock Alternating minimization as sequential unconstrained minimization: a
  survey.
\newblock {\em Journal of Optimization Theory and Applications}, 156:554--566,
  2013.

\bibitem[CT06]{cover1999elements}
Thomas~M. Cover and Joy~A. Thomas.
\newblock {\em Elements of Information Theory}.
\newblock John Wiley \& Sons, second edition, 2006.

\bibitem[DLR77]{dempster1977maximum}
Arthur~P. Dempster, Nan~M. Laird, and Donald~B. Rubin.
\newblock Maximum likelihood from incomplete data via the {EM} algorithm.
\newblock {\em Journal of the Royal Statistical Society Series B: Statistical
  Methodology}, 39(1):1--38, 1977.

\bibitem[HMLZ15]{hastie2015matrix}
Trevor Hastie, Rahul Mazumder, Jason Lee, and Reza Zadeh.
\newblock Matrix completion and low-rank {SVD} via fast alternating least
  squares.
\newblock {\em Journal of Machine Learning Research}, 16:3367--3402, 2015.

\bibitem[JCP{\etalchar{+}}24]{josse2024consistency}
Julie Josse, Jacob~M Chen, Nicolas Prost, Gaël Varoquaux, and Erwan Scornet.
\newblock On the consistency of supervised learning with missing values.
\newblock {\em Statistical Papers}, 65(9):5447--5479, 2024.

\bibitem[KSST12]{kakade2012regularization}
Sham~M Kakade, Shai Shalev-Shwartz, and Ambuj Tewari.
\newblock Regularization techniques for learning with matrices.
\newblock {\em Journal of Machine Learning Research}, 13:1865--1890, 2012.

\bibitem[LMV25]{morvan2024imputation}
Marine Le~Morvan and Ga{\"e}l Varoquaux.
\newblock Imputation for prediction: beware of diminishing returns.
\newblock In {\em International Conference on Learning Representations}, 2025.

\bibitem[L{\'o}c20]{loczi2020explicit}
Lajos L{\'o}czi.
\newblock Explicit and recursive estimates of the {Lambert} {W} function.
\newblock {\em arXiv preprint arXiv:2008.06122}, 2020.

\bibitem[LR19]{little2019statistical}
Roderick J.~A. Little and Donald~B. Rubin.
\newblock {\em Statistical Analysis with Missing Data}.
\newblock John Wiley \& Sons, third edition, 2019.

\bibitem[MFBR20]{marteau2020non}
Ulysse Marteau-Ferey, Francis Bach, and Alessandro Rudi.
\newblock Non-parametric models for non-negative functions.
\newblock {\em Advances in Neural Information Processing Systems},
  33:12816--12826, 2020.

\bibitem[MFBR22]{marteau2022sampling}
Ulysse Marteau-Ferey, Francis Bach, and Alessandro Rudi.
\newblock Sampling from arbitrary functions via {PSD} models.
\newblock In {\em Proceedings of The 25th International Conference on
  Artificial Intelligence and Statistics}, pages 2823--2861, 2022.

\bibitem[MFBR24]{marteau2024second}
Ulysse Marteau-Ferey, Francis Bach, and Alessandro Rudi.
\newblock Second order conditions to decompose smooth functions as sums of
  squares.
\newblock {\em SIAM Journal on Optimization}, 34(1):616--641, 2024.

\bibitem[MJBC20]{muzellec2020missing}
Boris Muzellec, Julie Josse, Claire Boyer, and Marco Cuturi.
\newblock Missing data imputation using optimal transport.
\newblock In {\em International Conference on Machine Learning}, pages
  7130--7140, 2020.

\bibitem[MTT{\etalchar{+}}20]{ma2020vaem}
Chao Ma, Sebastian Tschiatschek, Richard Turner, Jos{\'e}~Miguel
  Hern{\'a}ndez-Lobato, and Cheng Zhang.
\newblock {VAEM}: A deep generative model for heterogeneous mixed type data.
\newblock {\em Advances in Neural Information Processing Systems},
  33:11237--11247, 2020.

\bibitem[Mur18]{Murray2018}
Jared~S. Murray.
\newblock Multiple imputation: A review of practical and theoretical findings.
\newblock {\em Statistical Science}, 33(2):142--159, 2018.

\bibitem[NN94]{nesterov1994interior}
Yurii Nesterov and Arkadii Nemirovskii.
\newblock {\em Interior-{P}oint {P}olynomial {A}lgorithms in {C}onvex
  {P}rogramming}.
\newblock SIAM, 1994.

\bibitem[NSJ26]{naf:hal-04521894}
Jeffrey N{\"a}f, Erwan Scornet, and Julie Josse.
\newblock What is a good imputation under {MAR} missingness?
\newblock {\em Statistical Science}, 2026.
\newblock In press. arXiv:2403.19196.

\bibitem[NSMM23]{naf2023imputation}
Jeffrey N{\"a}f, Meta-Lina Spohn, Loris Michel, and Nicolai Meinshausen.
\newblock Imputation scores.
\newblock {\em The Annals of Applied Statistics}, 17(3):2452--2472, 2023.

\bibitem[NW06]{wright2006numerical}
Jorge Nocedal and Stephen~J. Wright.
\newblock {\em Numerical Optimization}.
\newblock Springer Series in Operations Research and Financial Engineering.
  Springer, second edition, 2006.

\bibitem[OR64]{olkin1964multivariate}
Ingram Olkin and Herman Rubin.
\newblock Multivariate beta distributions and independence properties of the
  {Wishart} distribution.
\newblock {\em The Annals of Mathematical Statistics}, 35(1):261--269, 1964.

\bibitem[PVG{\etalchar{+}}11]{pedregosa2011scikit}
Fabian Pedregosa, Ga{\"e}l Varoquaux, Alexandre Gramfort, Vincent Michel,
  Bertrand Thirion, Olivier Grisel, Mathieu Blondel, Peter Prettenhofer, Ron
  Weiss, Vincent Dubourg, et~al.
\newblock Scikit-learn: Machine learning in {P}ython.
\newblock {\em Journal of Machine Learning Research}, 12:2825--2830, 2011.

\bibitem[RC21]{rudi2021psd}
Alessandro Rudi and Carlo Ciliberto.
\newblock {PSD} representations for effective probability models.
\newblock {\em Advances in Neural Information Processing Systems},
  34:19411--19422, 2021.

\bibitem[RMFB25]{rudi2020finding}
Alessandro Rudi, Ulysse Marteau-Ferey, and Francis Bach.
\newblock Finding global minima via kernel approximations.
\newblock {\em Mathematical Programming}, 209(1--2):703--784, 2025.

\bibitem[SB12]{stekhoven2012missforest}
Daniel~J Stekhoven and Peter Bühlmann.
\newblock {MissForest}—non-parametric missing value imputation for mixed-type
  data.
\newblock {\em Bioinformatics}, 28(1):112--118, 2012.

\bibitem[SG05]{NIPS2005_4491777b}
Edward Snelson and Zoubin Ghahramani.
\newblock Sparse {G}aussian processes using pseudo-inputs.
\newblock In Y.~Weiss, B.~Sch\"{o}lkopf, and J.~Platt, editors, {\em Advances
  in Neural Information Processing Systems}, volume~18, pages 1257--1264, 2005.

\bibitem[SR17]{szekely2017energy}
G{\'a}bor~J Sz{\'e}kely and Maria~L Rizzo.
\newblock The energy of data.
\newblock {\em Annual Review of Statistics and Its Application}, 4(1):447--479,
  2017.

\bibitem[SSBD14]{shalev2014understanding}
Shai Shalev-Shwartz and Shai Ben-David.
\newblock {\em Understanding Machine Learning: From Theory to Algorithms}.
\newblock Cambridge University Press, 2014.

\bibitem[TSSE21]{tashiro2021csdi}
Yusuke Tashiro, Jiaming Song, Yang Song, and Stefano Ermon.
\newblock {CSDI}: Conditional score-based diffusion models for probabilistic
  time series imputation.
\newblock {\em Advances in Neural Information Processing Systems},
  34:24804--24816, 2021.

\bibitem[VBGO11]{van2011mice}
Stef Van~Buuren and Karin Groothuis-Oudshoorn.
\newblock mice: Multivariate imputation by chained equations in {R}.
\newblock {\em Journal of Statistical Software}, 45(3):1--67, 2011.

\bibitem[Vil09]{villani2009optimal}
Cédric Villani.
\newblock {\em Optimal Transport: Old and New}.
\newblock Springer, 2009.

\bibitem[WAPL22]{wang2021deep}
Zhenhua Wang, Olanrewaju Akande, Jason Poulos, and Fan Li.
\newblock Are deep learning models superior for missing data imputation in
  surveys? {E}vidence from an empirical comparison.
\newblock {\em Survey Methodology}, 48(2):375--399, 2022.

\bibitem[YJvdS18]{yoon2018gain}
Jinsung Yoon, James Jordon, and Mihaela van~der Schaar.
\newblock {GAIN}: Missing data imputation using generative adversarial nets.
\newblock In {\em International Conference on Machine Learning}, pages
  5689--5698, 2018.

\end{thebibliography}

\clearpage
\appendix

{\Large
\begin{itemize}
\item Experiments -- \Cref{sec:experiments}. Toy and real data experiments.
\item Theory -- \Cref{appx:theory}. Theoretical analysis of the algorithm.
\item Lemmas -- \Cref{appx:lemmas}. Technical lemmas needed in the theoretical part.
\item Computations and Implementation -- \Cref{appx:comput}. Detailed analysis of the algorithm.

\end{itemize}
}
\vspace{0.5cm}

\section*{Notation Index}
\label{appx:notation}
\begin{center}
\begin{tabular}{ll}
$X$ & full data vector in $\mathbb R^d$\\
$S$ & random subset of observed coordinates\\
$\phi(x)$ & kernel feature map centred at anchors\\
$Q$ & PSD parameter matrix\\
$H$ & moment matrix $\int\phi\phi^T$\\
$A_i$ & data-dependent kernel matrix\\
$\mu$ & barrier weight in optimisation\\
\end{tabular}
\end{center}

\section{Experiments}
\label{sec:experiments}

\subsection{Experimental setup}\label{subsec:setup}

Our empirical study is exploratory. We do not aim to demonstrate that our method is superior, as further work is needed to develop code suitable for practical use. Indeed, at this stage the algorithm does not scale to high dimensions or large datasets. Instead, we aim to illustrate its potential through exploratory experiments, leaving comprehensive benchmarking and large-scale engineering to future work. The code was written with the support of an AI tool.

In our experiments, we used a Gaussian kernel and integrated it over the entire ambient space and we have fixed the level $\alpha=10^{-6}$ to reduce computational cost.

We benchmark our method against theory-based methods in \cite{muzellec2020missing}, using their publicly available implementation as well as their training and evaluation setup. As a simple baseline, we use coordinate-wise \textbf{mean imputation}. \texttt{IterativeImputer} \cite{pedregosa2011scikit} is inspired by the well known MICE software \cite{van2011mice} and consists of iterative imputation by conditional mean expectation. \textbf{SoftImpute} \cite{hastie2015matrix} is a low-rank method that alternates between solving least-squares problems and applying singular value thresholding. \textbf{OT} \cite{muzellec2020missing} is an optimal transport based method, which minimizes a suitable metric constructed using the Sinkhorn divergence. \textbf{Linear RR} and \textbf{MLP RR} \cite{muzellec2020missing} (RR stands for Round--Robin) are feature-by-feature imputation methods that mimics \textsc{IterativeImputer} or \textsc{mice}. The first method fits a linear model with respect to the $d-1$ other variables and it is fitted with the OT loss equation. The second method is similar and the imputer is a Multi Layer Perceptron.

Our method \textbf{PSD\_Impute} appears as a conditional mean estimator $psd\_mean$ and as a multiple imputer $psd\_multiple$ with $m=10$ draws. 
We perform the computation of the conditional mean in Section \ref{app:psd_conditional_mean_imputation_psd}, and sampling from a PSD model follows \cite{marteau2022sampling}. %

\textbf{Metrics.} Energy Distance (ED) \cite{szekely2017energy} is our primary criterion because it compares full joint distributions, %
and RMSE \cite{muzellec2020missing} is reported only to verify that distributional gains do not harm point accuracy. Error bars show one standard deviation around the mean over the five random seeds.

\textbf{Train/Test splitting and cross--validation procedure.} 
We run the experiments for $5$ repetitions; datasets with more than $1000$ observations are split into $800$ training rows and a test set made of the remaining rows. Smaller datasets are split as $70\%$ training and $30\%$ test. Then, $20\%$ of the entries of the training set are hidden completely at random, and the hyperparameters are re--selected at each repetition using only $X_{train}^{(r)}$; the test set is never involved. The final training set is $(X_{train}^{(r)}, M_{train}^{(r)})$, with $M_{train}^{(r)}$ a mask that denotes the $20\%$ missing entries, and the final hold--out test set is $X_{test}^{(r)}$, for $r=1,2,3,4,5.$
We draw 3 random partitions of the rows of $(X_{train}^{(r)},M_{train}^{(r)})$ into $100$ fitting rows $(X_{sub\_train}^{(r,h)},M_{sub\_train}^{(r,h)})$ and
a validation set made of all remaining rows $(X_{val}^{(r,h)},M_{val}^{(r,h)})$ for $h=1,2,3$. In the validation rows, each entry that is still observed is additionally hidden with probability $0.2$, keeping at
least one observed entry per row, so that we get a new mask $\tilde M_{val}^{(r,h)}$ such that $M_{val}^{(r,h)}\subseteq \tilde M_{val}^{(r,h)}.$ 
From the $100$ fitting rows we draw 4 random candidate subsets of $\ell = 65$ anchor nodes, say $(X_{sub\_train}^{(r,h,k)},M_{sub\_train}^{(r,h,k)})$ for $k=1,2,3,4$, with their missing entries set to 0.
We search jointly over the initial log-precision $\eta_0 \in \{0.4, 0.6, 0.8, 1.2\}$ and
the initial anchor nodes, with the learning rates fixed to $3 \cdot 10^{-4}$.
For every pair of $\eta_0$ and anchor set $(X_{sub\_train}^{(r,h,k)},M_{sub\_train}^{(r,h,k)})$, we run the algorithm and then we impute the entries in $\tilde M_{val}^{(r,h)}\setminus M_{val}^{(r,h)}$ with the conditional
mean, and compute the RMSE on the extra-hidden entries, whose true values are
known. This amounts to $3 \times 4 \times 4 = 48$ fits per repetition. The pair
with the lowest RMSE is selected, and the final model is fitted on the whole of
$(X_{train}^{(r)},M_{train}^{(r)})$ starting from its $\eta_0$ and anchor nodes with a few iterations.

\textbf{Implementation details.} Experiments were conducted on an institutional HPC cluster equipped with Intel Xeon (Cascade Lake, Skylake, and Broadwell) processors. The total computational time is reported in Figure \ref{fig:runtime}. Newton iterations were computed using the SMW formula, as the datasets were of moderate size.

\subsection{Synthetic toy problems}\label{subsec:toy}
\begin{figure}[t]
  \centering
  \includegraphics[width=1\linewidth]{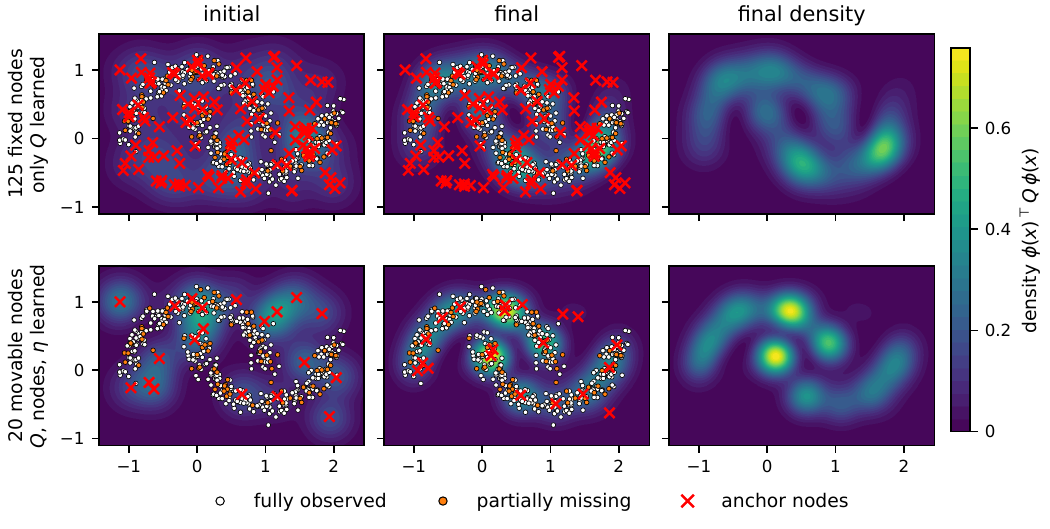}%
  \caption{Qualitative results on 2--D toy datasets.  Left: input points (missing coordinates projected). Top: $125$ anchors sampled $i.i.d.$ uniformly randomly, bottom: $20$ anchors sampled $i.i.d.$ uniformly randomly. Center: Learned density with final position of nodes. Top: the positions haven't changed, bottom: the non--convex optimization yields new anchor positions. Right: final density learned, without the position of the nodes on display.} 
  \label{fig:toy}
\end{figure}

PSD--Impute faithfully reconstructs a complex manifold (moons) with both the convex formulation with many nodes, and when we use fewer nodes with the non--convex formulation. See Figure \ref{fig:toy}. This image shows that we can get good results also with a moderate number of anchor nodes, if their position is meaningful, as discussed in Section \ref{sec:comp}. We report a peak of around 115 MB of memory consumption for the bigger dataset, and 68.1 MB for the small one, for our method, and we expect it to be worse than the other methods.

\subsection{Real life datasets}
\begin{figure}[t]
\centering
\includegraphics[width=1\linewidth]{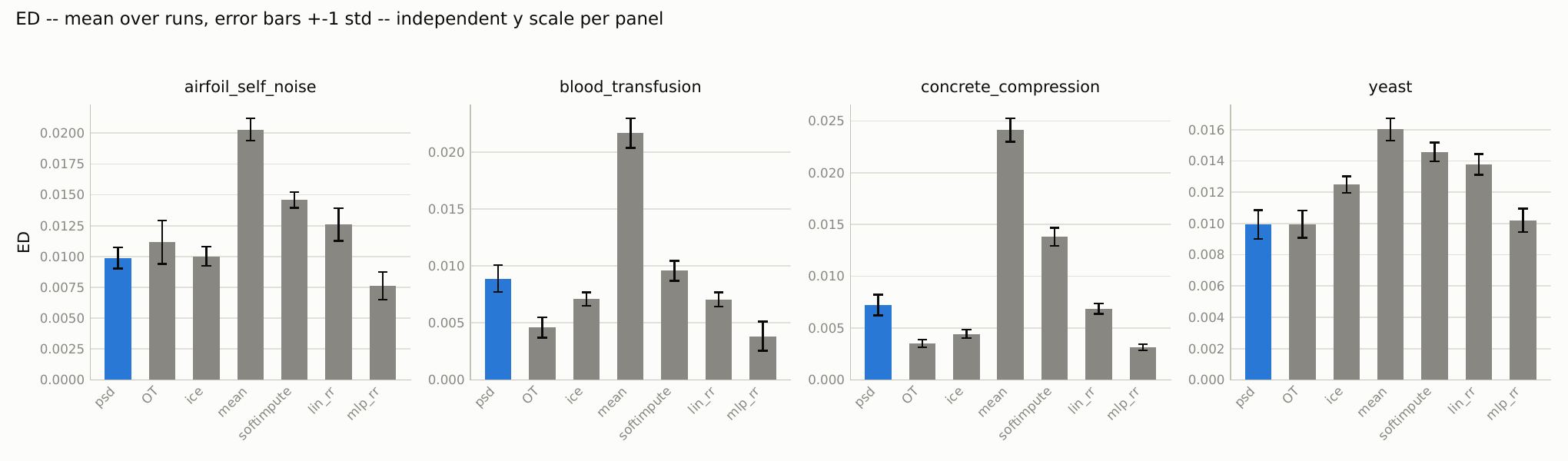}
\includegraphics[width=1\linewidth]{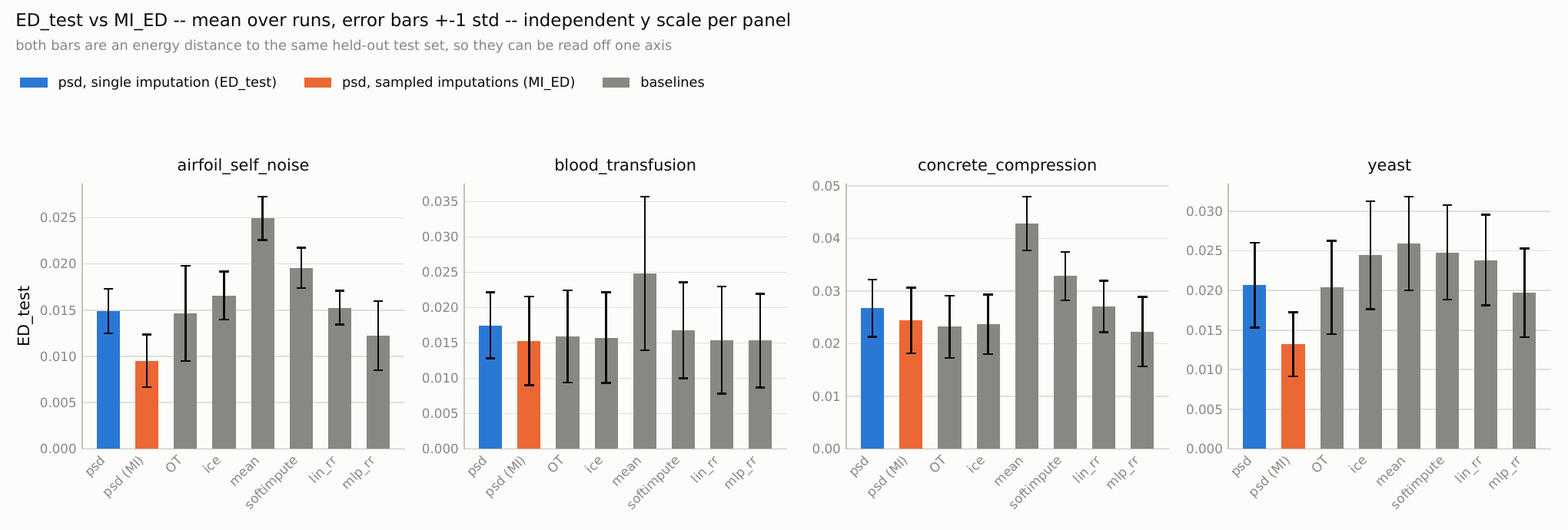}
\caption{\textbf{Distributional accuracy on real datasets.}  Top: our PSD method imputes missing values in the training set using conditional means and the completed dataset is compared with the ground truth. Bottom: the imputed dataset is compared with a held-out test set. The orange bar corresponds to drawing $10$ samples per imputed point from the conditional PSD model, and it provides a sanity check of whether the learned conditional density captures meaningful structure. We remark that, unless we have identifiability, this is not a decisive hint that the learned joint density is the true underlying density, see Section \ref{sec:comp}.
Because the number of imputed points is ten times that of a single imputation, its results are not directly comparable with those of the other methods.}
\label{fig:dist_metrics_ed}
\end{figure}

\begin{figure}[h]
\centering
\includegraphics[width=1\linewidth]{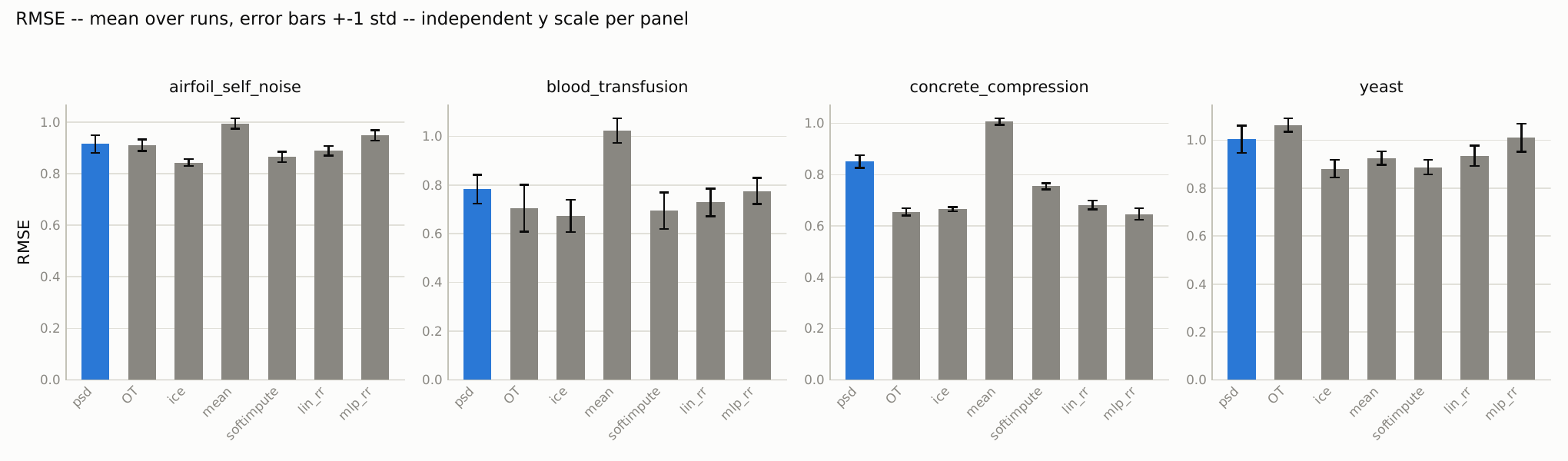}%
\caption{%
Lower is better, RMSE on the $y$ axis.
Our method is comparable in the computation of the conditional mean on three out of four datasets.}
\label{fig:rmse-figs}
\end{figure}

\begin{figure}[h]
\centering
\includegraphics[width=1\linewidth]{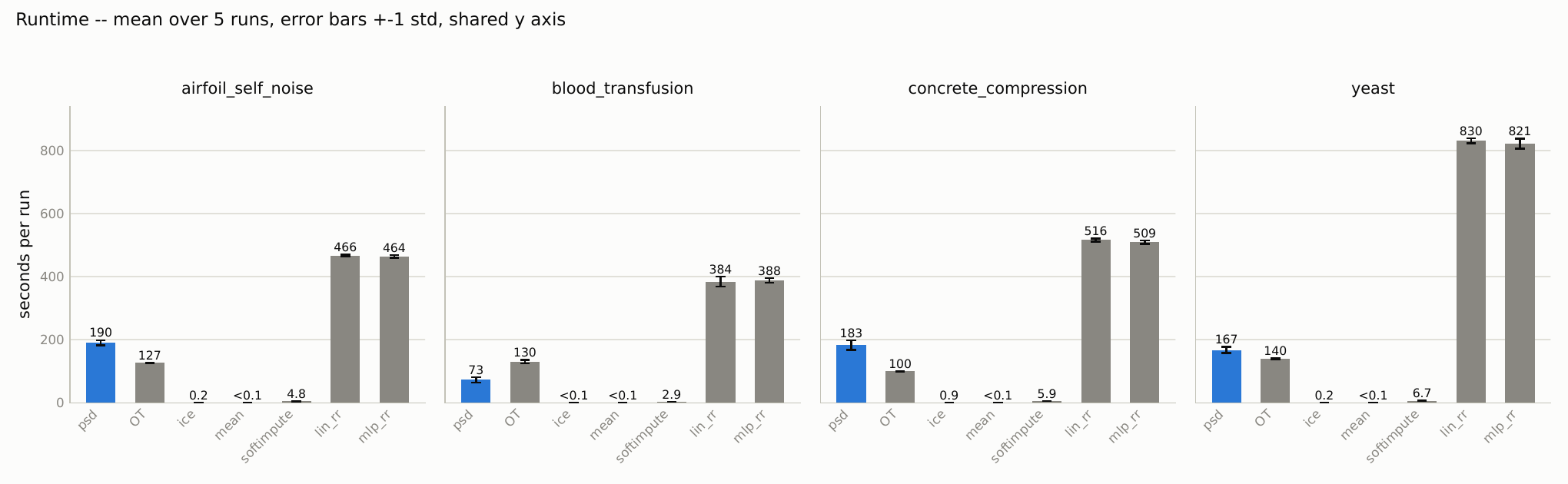}%
\caption{%
Runtime across experiments on the real--world datasets. Our method has a runtime of the same order of magnitude as the OT method and is faster than the two OT iterative imputation methods.}
\label{fig:runtime}
\end{figure}

See the results in Figure \ref{fig:dist_metrics_ed} for a comparison of the methods with respect to the energy distance, on the training set and the hold--out test set. See Figure \ref{fig:rmse-figs} for the same result with respect to the RMSE and see Figure \ref{fig:runtime} for a report of the runtime. We tested our method against the baselines on four datasets from the UCI machine learning repository used in \cite{muzellec2020missing}, namely $airfoil\_self\_noise$ of size $n=1503$ and dimension $d=5$, then $blood\_transfusion$ of size $n=748$ and dimension $d=4$, then $concrete\_compression$ of size $n=1030$ and dimension $d=7$, and finally $yeast$ of size $n=1484$ and dimension $d=8.$

\subsection{Discussion and limitations}\label{subsec:limit}
\begin{itemize}
  \item \textbf{Strengths} – 
  The method has a strong theoretical foundation, rooted in the theory of kernels and optimization. The explicit density enables single \emph{and} multiple imputation, and the results are competitive.
  \item \textbf{Limitations} – The algorithm has many hyperparameters that must be fine-tuned, notably the anchor nodes and the precision of the Gaussian. The implementation is on a CPU architecture, and the algorithm suffers from the dependence on many hyper‐parameters. The algorithm does not scale well with the number of anchor nodes, and the heuristic might become expensive. 

  \item \textbf{Outlook} – In future work, we plan to scale the number of anchor points $\ell$ and the ambient dimension $d$ by leveraging Nyström sketching. 
  We also plan to investigate more efficient hyper‐parameter search strategies, as well as suitable implementations for MAR and MNAR extensions.

\end{itemize}

\clearpage

\section{Theoretical Guarantees}
\label{appx:theory}
\subsection{Proof of Thm.~\ref{thm:exists-barQ}}
\label{appx:proof-exists_barQ}

\begin{proof}
Let ${\cal W}$ be a mesh of $\ell$ points satisfying assumption (ii), then it has a {\em fill distance} $h := \max_{x \in D} \min_{w \in W} \|x-w\|$ bounded by
$$
 h \leq q_1 \left(\frac{\log(\ell/\delta)}{\ell}\right)^{1/d},
$$
with probability $1-\delta$, by Lemma~4 of \cite{rudi2020finding} and the fact that the cube $D=(-1,1)^d$ is contained in the ball of radius $2\sqrt{d}$, where $q_1$ depends only on $d$.
Let $C_1 = \max\{3, (q_1\sqrt{C}/c')^{d}(\max(1,d/s) + d/2 + 1)\}$ and $C_2 = C$, where $C,c'$ are defined in \cref{thm:exists-barQ-general-h}.
Set $L=\log(2/\epsilon)$ and $\Lambda=\log(C_1/(\epsilon\delta))$. 

Plugging in $\eta$, it follows that we only need to check that
$\ell/\log(\ell/\delta)\geq (q_1\sqrt{C}/c')^{d}\epsilon^{-d/s}L^{d/2}$. 

The function $\ell\mapsto \ell/\log(\ell/\delta)$ is increasing for $\ell\geq e\delta$, and $\ell_0=C_1\epsilon^{-d/s}\Lambda^{d/2+1}\geq 3> e\delta$, since $C_1\geq 3$ and $\Lambda\geq\log 3\geq 1$, so it is enough to check the condition at $\ell=\ell_0$. Using $L\leq\Lambda$ and $\log\Lambda\leq\Lambda$, we have $\log(\ell_0/\delta)=\log C_1+\frac{d}{s}\log\frac{1}{\epsilon}+(\frac{d}{2}+1)\log\Lambda+\log\frac{1}{\delta}\leq(\max(1,d/s)+\frac{d}{2}+1)\Lambda$, so that
$$
\frac{\ell_0}{\log(\ell_0/\delta)}\geq\frac{C_1\,\epsilon^{-d/s}\Lambda^{d/2}}{\max(1,d/s)+d/2+1}\geq\Big(\frac{q_1\sqrt{C}}{c'}\Big)^{d}\epsilon^{-d/s}L^{d/2}.
$$
Hence, by the choice of $\ell$ and the result above, we have that $h \leq \frac{c'}{\sqrt{\eta}\log(2/\epsilon)}$ with probability $1-\delta$, so we can apply \cref{thm:exists-barQ-general-h} obtaining the existence of a $\bar{Q}$ such that $\|p_{\bar{Q}} - p^*\|_{L^2(D)} \leq \epsilon$ and $\tr(\bar{Q}K_\eta) \leq C' \epsilon^{-d/s} \leq C_3 \ell$, with $C_3 = C'/C_1$, where $C'$ is in \cref{thm:exists-barQ-general-h}.
\end{proof}

\subsection{Proof of (Thm.~\ref{thm:exists-barQ-general-h})}
\label{appx:proof-exists-barQ-general_h}
The following is the deterministic result that allows us to prove Theorem (\ref{thm:exists-barQ}).
\begin{theorem}\label{thm:exists-barQ-general-h}
Let $D = (-1,1)^d$ be the $d$--dimensional hypercube.
Let $f_1,\dots, f_q \in W^s_2(\R^d) \cap L^\infty(\R^d)$ and let $p = \sum_{i=1}^q f_i^2$ be a probability on $D$.
Let ${\cal W}=\{w_i\}_{i=1,\dots,\ell}$ be a set of points in $D$ with {\em fill distance} $h := \max_{x \in D} \min_{w \in {\cal W}} \|x-w\|$.
Then, for any $\epsilon \in (0, \epsilon_0]$, with $\epsilon_0>0$ a $p$--dependent constant, by setting $(\eta_i)_{i=1,\dots d} \in \R^d_{++}$, as
$$\eta_i = \frac{C}{\epsilon^{2/s}\log\frac{2}{\epsilon}}, \quad h \leq \frac{c'}{\sqrt{\eta}\log(2/\epsilon)},$$
there exists a matrix $\bar{Q} \in \R^{\ell\times \ell}$, $Q \succeq 0$, such that 
$$\tr(\bar{Q} K_{\eta}) \leq C' \epsilon^{-d/s}, \quad \|p_{\bar{Q}} - p\|_{L^2(D)} \leq \epsilon,$$
where $c',C,C'$ depend only on $s, d, \|f_j\|_{W^s_2(\R^d)}, \|f_j\|_{L^\infty(\R^d)}$ for $j=1,\dots,q$. 
\end{theorem}

\begin{proof}
Let ${\cal H}_\eta$ be the Reproducing Kernel Hilbert space associated to the Gaussian kernel $k_\eta$ on $\R^d$ with width $\eta$. Let $\scal{\cdot}{\cdot}$ be the associated inner product and denote by $k_x \in {\cal H}_\eta$ the associated reproducing vector for the point $x \in \R^d$. See the appendix of \cite{rudi2021psd} for more details.
Let $\alpha = \epsilon/2$ and $\eta = \tau {\bf 1}_d \in \R^d$ for some $\tau > 0$ and $\ell \in \N$. Let $M_\alpha: {\cal H}_\eta \to {\cal H}_\eta$ be the operator constructed in Thm D.4 of \cite{rudi2021psd} to approximate $p$. Now let ${\cal W} = \{w_1,\dots, w_\ell\}$ be a covering of $D$.
We consider the following model ${p}_\ell = p_{\bar{Q}}$ where 
$$\bar{Q} ~=~ K_{\eta}^{-1} \,B_\alpha\, K_{\eta}^{-1},$$
where $(B_\alpha)_{ij} = \scal{k_{w_i}}{M_\alpha k_{w_j}}$ and ${K_{\eta}}_{ij} = k_\eta(w_i, w_j)$. Observe that $\int_D p_{\ell}(x)dx=\tr(\bar{Q}H),$ with $H_{i,j}=\int_{D} k_{\eta}(x,w_i)k_{\eta}(x,w_j)dx.$

\paragraph{Step 1. Approximation error decomposition.}
We will split the approximation error as follows,
$$
\|p_\ell - p\|_{L^2(D)} ~\leq~ \|p_\ell - p_\alpha \|_{L^2(D)} +  \|p_\alpha - p\|_{L^2(D)}. 
$$
Note that for the second term, by Thm D.4 of \cite{rudi2021psd} if $0<\alpha=\epsilon/2\leq 1$, we have
$$
\|p_\alpha - p\|_{L^r(D)} \leq \|p_\alpha - p\|_{L^r(\R^d)} \leq \alpha, \qquad \forall~r \in [1,2]$$

\paragraph{Step 2. Error induced by a projection.}
By Thm C.4 of \cite{rudi2021psd} applied to $M_\alpha$, we have that when the fill distance $h$ satisfies $h \leq \sigma/C'_d$ with $\sigma = \min(1,1/\sqrt{\tau})$, then
$$|p_\alpha(x) - p_{\bar{Q}}(x)| ~\leq~ 2C_d \sqrt{\|M_\alpha\| p_\alpha(x)}\, e^{-\frac{c_d\,\sigma}{h}\log\frac{c_d\,\sigma}{h}}  \,+\, 4C_d^2\|M_\alpha\|\,e^{-\frac{2c_d\,\sigma}{h}\log\frac{c_d\,\sigma}{h}}, \quad \forall x \in D
$$
with $c_d, C_d, C'_d$ depending only on $d$. Now denoting by $a = 2C_d \sqrt{\|M_\alpha\|} e^{-\frac{c_d\,\sigma}{h}\log\frac{c_d\,\sigma}{h}}$ %
we have
$$\|p_\alpha - p_{\bar{Q}}\|_{L^2(D)} \leq a \|\sqrt{p_\alpha}\|_{L^2(D)} + a^2 \|{\bf}1\|_{L^2(D)} \leq 2^{d/2}a^2 + a \|p_\alpha^{1/2}\|^2_{L^2(D)}.$$

\paragraph{Step 3. Bound of $\tr(\bar{Q} K_{\eta})$.}
Note that we have that $\tr(\bar{Q} K_{\eta}) \leq \tr(M_\alpha)$, since $B_\alpha = Z M_\alpha Z^*$, and $K_{\eta} = Z Z^*$ with $Z$ the linear operator $Z: {\cal H}_\eta \to \R^\ell$ defined as $Z f = (f(w_1),\dots f(w_\ell))$ for any $f \in {\cal H}_\eta$. So, by definition of $K_\eta, B_\alpha$ and the cyclicity of the trace
\begin{align*}
\tr(\bar{Q} K_{\eta}) &= \tr(K_\eta^{-1} Z M_\alpha Z^*) = \tr( Z^* (Z Z^*)^{-1} Z M_\alpha) \\
& \leq \|Z^* (Z Z^*)^{-1} Z\| \tr(M_\alpha) \leq \tr(M_\alpha)
\end{align*}

\paragraph{Step 4. Final bound.}
First, note that by Theorem D.4 of \cite{rudi2021psd}
$$
\tr(M_\alpha) \leq  C_1 \tau^{d/2}(1 + \alpha^2\exp(\tfrac{C_2}{\tau} \alpha^{-\frac{2}{s}}))
$$
where $C_1, C_2$ are independent of $\alpha, \tau$ and depend only on $s, d, \|f_i\|_{W^s_2(\R^d)}, \|f_i\|_{L^\infty(\R^d)}$. When $\tau \geq C_2/(2 \alpha^{2/s} \log(1/\alpha))$ we have
$$
\|M_\alpha\| \leq \tr(M_{\alpha}) \leq  C_1 \tau^{d/2}(1 + \alpha^2\exp(\tfrac{C_2}{\tau} \alpha^{-\frac{2}{s}})) \leq  C_3 \alpha^{-d/s}.
$$
with $C_3 = 2^{2-d/2} C_1 C_2^{d/2}$. Then, note that $\|p_\alpha^{1/2}\|_{L^2(D)} = \|p_\alpha\|^{1/2}_{L^1(D)}$, so 
we have
$$
\|p_\alpha\|_{L^1(D)} \leq \|p_\alpha - p\|_{L^1(D)} + \|p\|_{L^1(D)} \leq 1 + \alpha \leq 2.
$$
Set $t=c_d\sigma/h$. If $t=t_{\alpha}\geq \max(C'_d, (1+\frac{d}{2s}) \log \frac{1}{\alpha} \, + 1+\frac{d}{4} \, + \log(C_d \sqrt{C_3}) + e)$ then $t\geq e$ and $\log t\geq 1,$ so 
$$
C_d e^{-\frac{c_d\,\sigma}{h}\log\frac{c_d\,\sigma}{h}} = C_de^{-t \log t} \leq C_de^{-t} \leq \tfrac{2^{-d/4}}{15\sqrt{C_3}} \alpha^{1 + \frac{d}{2s}}.
$$
Gathering the results from the previous steps, we have
\begin{align*}
\|p_l - p\|_{L^2(D)} &\leq \alpha  + 4 \|M_\alpha\|^{1/2} C_d e^{-\frac{c_d\,\sigma}{h}\log\frac{c_d\,\sigma}{h}}  ~+~ \|M_\alpha\| C_d^2 e^{-\frac{c_d\,\sigma}{h}\log\frac{c_d\,\sigma}{h}}\\
& \leq \alpha + \tfrac{2^{2-(d/4)}}{15}\alpha + \tfrac{2^{-d/2}}{225}  \alpha^2 \leq 2\alpha.
\end{align*}

The final step is obtained by setting $h=c_d\sigma/t_{\alpha},$ with $t_{\alpha}\geq C'_d\cdot c_d,$ so as to get $h\leq \sigma/C'_d,$ and $\tau=\tau_{\alpha}=C_2\alpha^{-2/s}\log(1/\alpha)/2.$ With these choices we satisfy the two conditions required on $\tau$ for Thm C.4, D.4 of \cite{rudi2021psd} and the last bound to hold.

\paragraph{Step 5. Renormalization of the density.} Now, to make sure that our model is a proper density, we need to renormalize it so that it integrates to one. Let us consider $p_{\ell}/\int_{D}p_{\ell}dx=p_{\ell}/\tr(\bar Q H).$ Then
\begin{align*}
    \|p_{\ell}/\tr(\bar Q H)-p\|_{L^2(D)}
    &\leq \frac{1}{\tr(\bar Q H)}\| p_{\ell} -p\|_{L^2(D)} + \left|1-\frac{1}{\tr(\bar Q H)}\right|\|p\|_{L^2(D)}\\
    &\leq \frac{1}{\tr(\bar Q H)}\| p_{\ell} -p\|_{L^2(D)} + \left|\frac{\tr(\bar Q H)-1}{\tr(\bar Q H)}\right|\|p\|_{L^2(D)}
\end{align*}
We have that
\begin{align*}
    |\tr(\bar QH)-1|=\left|
    \int \left(p_{\ell}(x)-p(x)\right)dx
    \right|\leq 
    \|p_{\bar Q}-p\|_{L^1(D)}
    \leq V_D^{1/2}\|p_{\bar Q}-p\|_{L^2(D)}\leq V_D^{1/2}\varepsilon=\varepsilon'
\end{align*}
that implies
\begin{align*}
    \frac{1}{1+\varepsilon'}\leq \frac{1}{\tr(\bar QH)}\leq \frac{1}{1-\varepsilon'}.
\end{align*}
So we deduce that
\begin{align*}
    \|p_{\ell}/\tr(\bar Q H)-p\|_{L^2(D)}\leq 
    \frac{\varepsilon}{1-\varepsilon'} + \frac{\varepsilon'}{1-\varepsilon'}\|p\|_{L^2(D)}
    \leq \frac{\varepsilon'}{1-\varepsilon'}
    \left(1+\|p\|_{L^2(D)} \right)
\end{align*}
where we have used that $\|p_{\ell}-p\|_{L^2(D)}\leq \varepsilon\leq \varepsilon'$, since $D$ is the unit cube. If we take %
$\varepsilon'=V_D^{1/2}\varepsilon\geq 1/2$
then for $\epsilon\in(0,\epsilon_0]$ we have the final bound 
\begin{align*}
    \|p_{\ell}/\tr(\bar Q H)-p\|_{L^2(D)}\leq \epsilon \varepsilon V_D^{1/2}\left(
        1 + \|p\|_{L^2(D)}
    \right),
\end{align*}
and this can be made smaller than a pre--chosen $\varepsilon>0$ if we start with $\varepsilon/(V_D^{1/2}(1 + \|p\|_{L^2(D)}))$ as we wanted. In particular we need $\varepsilon\in (0,\varepsilon_0]$, with $\varepsilon_0=1/(2V_D).$

\end{proof} 

\subsection{Proof of (Thm.~\ref{thm:approx-error-tildeQ})}
\label{appx:proof-approx-error-tildeQ}
\begin{proof}
Let $\bar{Q}$ be the matrix defined in \cref{thm:exists-barQ}. It satisfies with probability $1-\delta$ the bounds 
$\|p_{\bar{Q}} - p^*\|_{L^2(D)} \leq \epsilon$ and $\tr(\bar{Q} K_\eta) \leq C_3 \ell$ for some constant $C_3$ defined in the theorem.
Since $\tilde{Q}$ is the minimizer of $\tilde{\cal L}$, we have that
$$
 \tilde{\cal L}(\tilde{Q}) \leq\tilde{\cal L}(\bar{Q}) = {\cal L}((1-\nu)p_{\bar{Q}} + \nu u) + \lambda \tr(\bar{Q} K_\eta).
$$
In particular, note that, by the information monotonicity property of KL (see e.g. \cite{cover1999elements}, Theorem 2.5.3 on the chain rule for relative entropy)  we have that for any density $p$ on $D$,
$${\cal L}(p) = \int KL(p^*_S \| p_S) d\mu(S) \leq KL(p^* \| p),$$
moreover, since $p^* \leq \|p^*\|_{L^\infty(D)} \leq C\|p^*\|_{W^s_2(D)} < \infty$ since $s > d/2$ by assumption and $C$ is the constant from Morrey's embedding \cite[Corollary 9.13, p. 283]{brezis2011functional}, and $\nu > 0$, by applying the {\em reverse Pinsker inequality} \cite{binette2019note}, we have
$${\cal L}((1-\nu)p_{\bar{Q}} + \nu u) \leq KL(p^*||(1-\nu)p_{\bar{Q}}+\nu u) \leq \log(C_2/\nu) \|(1-\nu)p_{\bar{Q}} + \nu u - p^*\|_{L^1(D)},$$
where $C_2 := e V_D \|p^*\|_{L^\infty(D)}$. We conclude by noting that
\begin{align*}
    \|(1-\nu)p_{\bar{Q}} + \nu u - p^*\|_{L^1(D)} 
    &\leq (1-\nu)\|p_{\bar{Q}} - p^*\|_{L^1(D)} + 2\nu\\
    &\leq (1-\nu)V_D^{1/2}\|p_{\bar{Q}} - p^*\|_{L^2(D)} + 2\nu\\ 
    &\leq 2(V_D^{1/2}\epsilon+\nu)\\
    &\leq 2(1+V_D^{1/2})(\epsilon+\nu)
\end{align*}
where in the second inequality we have used that $\|\cdot\|_{L^1(D)} \leq V_D^{1/2} \|\cdot\|_{L^2(D)}$, so that the final bound will be 
$${\cal L}((1-\nu)p_{\bar{Q}} + \nu u) 
\leq \log(C_2/\nu) \|(1-\nu)p_{\bar{Q}} + \nu u - p^*\|_{L^1(D)}
\leq 2(1+V_D^{1/2})\log(C_2/\nu)(\epsilon + \nu)
,$$

The desired result is obtained by putting $C_1 := 2 (1+V_D^{1/2})$.
\end{proof}

\subsection{Proof of (Thm.~\ref{thm:learn})}
\label{appx:proof_th_learn}
\begin{proof}
Let $\tau = \delta/2$.
By applying \cref{thm:approx-error-tildeQ} with $\epsilon := \nu$ we know that with the given choice of $\ell$ and $\eta$ we have that $\tilde{Q} \in \R^{\ell \times \ell}$, which is a feasible solution of Problem~\eqref{eq:primal}, satisfies with probability $1-\tau$
\begin{equation} \label{eq:bound-tildeQK_eta}
\tr(\tilde{Q} K_\eta) \leq (1/\lambda) {\tilde{\cal L}(\tilde{Q})} \leq C_3 \ell ~+~ (2C_1/\lambda) \log(C_2/\nu) \nu.
\end{equation}
Since the correct metric to analyze the problem is via a rescaling by the matrix $K_\eta$, we consider the equivalent problem to Problem~\eqref{eq:primal} where we perform the change of variable $\hat{M} = K_\eta^{1/2} \hat{Q} K_\eta^{1/2}$:
$$ \hat{M} = \min_{M \succeq 0, \tr(M C) = 1 } \Lambda(M) := - \frac{1}{N} \sum_{i=1}^n \log (\tr(M B_i) + \alpha_{S_i}) + \lambda\tr(M),$$
where we define $\alpha_{S_i}=\nu\int_{D_{NS_i}}1dx_{NS_i}/(V_D(1-\nu))$ and $C = K_\eta^{-1/2} H K_\eta^{-1/2}$ and $B_i = K_\eta^{-1/2} A_i K_\eta^{-1/2}$. Indeed, by the cyclicity of the trace we have
$\tr(MB_i) = \tr(QA_i)$, $\tr(M C) = \tr(Q H)$ and $\tr(M) = \tr(Q K_\eta)$.
So we define also $\tilde{M} = K_\eta^{1/2} \tilde{Q} K_\eta^{1/2}$ and more generally $\tilde{\cal L}(Q) = \Lambda(M)$ where $M = K_\eta^{1/2} Q K_\eta^{1/2}$.
Note that $\tilde{M}$ is feasible for the equivalent problem, due to the feasibility of $\tilde{Q}$, so
\begin{align*}
    -\frac{1}{N}\sum_{i=1}^N\log(\tr(\hat{M}) + \alpha_{S_i}) + \lambda \tr(\hat{M})
    &\leq    -\frac{1}{N}\sum_{i=1}^N\log(
        \|B_i\|_{op} \tr(\hat{M}) + \alpha_{S_i}) + \lambda \tr(\hat{M})\\
    &\leq \Lambda(\hat M)\\
    &\leq \Lambda(\tilde M)\\ 
    &\leq 
    \frac{1}{N}\sum_{i=1}^N\log(1/\alpha_{S_i}) + \lambda \tr(\tilde M)
\end{align*}
where in the first and second inequality we use that $0\leq\|B_i\|_{op}\leq 1,$ proven in \cref{lem:AinvK} and $-\log$ is a decreasing function, in the third we use the feasibility of $\hat{M}$ and $\tilde M,$ plus the fact that $\hat{M}$ is the minimizer of $\Lambda,$ and in the last equality we use that $\tr(\tilde M)$ is positive, so that it can be removed from the logarithm term. So, by moving the sum from the left part to the right and dividing by $\lambda$ we get
\begin{align*}
    \tr(\hat{M})
    &\leq \frac{1}{\lambda N}\sum_{i=1}^N \left(\log(\tr(\hat{M}) + \alpha_{S_i}) + \log(1/\alpha_{S_i})\right) + \tr(\tilde M)\\
    &\leq \frac{1}{\lambda N}\sum_{i=1}^N \log(\tr(\hat{M})/\alpha_{S_i} + 1)  + \tr(\tilde M)\\
    &\leq \frac{1}{\lambda}\log(\tr(\hat{M})/\alpha + 1)  + \tr(\tilde M),
\end{align*}
where we define $\alpha = \min_{S\subseteq[d]}\alpha_S.$ This last inequality implies, by Corollary \ref{corollary:W_Lambert_bound_corollary} the following

\begin{align*}
    \tr(\hat{M}) 
    &\leq \tr(\tilde{M}) + \frac{1}{\lambda} \log\left(3+ \frac{2}{\alpha} \tr(\tilde{M}) + \frac{2\log(1/(\lambda\alpha))}{\lambda\alpha}\right)\\ 
    &\leq C_3\ell +  2\frac{C_1}{\lambda}\log(C_2/\nu)\nu + \\
    &+ \frac{1}{\lambda} \log\left(3+ \frac{2}{\alpha}C_3\ell +  \frac{4C_1}{\alpha\lambda}\log(C_2/\nu)\nu + \frac{2\log(1/(\lambda\alpha))}{\lambda\alpha}\right) 
\end{align*}
Now, observe that for every $\nu>0,$ the quantity $\nu/\alpha$ is a constant that depends on the integration domain, so it is bounded. We deduce that we can bound the log--term outside and inside the logarithm
\begin{align*}
    &2\frac{C_1}{\lambda}\log(C_2/\nu)\nu\leq\frac{4C_1}{\alpha\lambda}\log(C_2/\nu)\nu\leq
    C_4
    \frac{1}{\lambda}
    \log\left(
    C_2/\nu 
    \right)%
\end{align*}
where we define $C_4=\frac{4C_1\nu}{\alpha} + 1$ and we are assuming $0<\lambda\leq 1$ and $0<\alpha\leq 1$ to simplify the equations.
Balancing the two log-terms, i.e. estimating them so that they have the same argument inside, we get that
\begin{align}
    \tr(\hat{M})\leq C_3\ell + \frac{C_4}{\lambda}
    \log\left(C_5\frac{\ell}{\nu}
    +\frac{2C_4}{\nu\lambda}\log\left(C_6/(\lambda\nu)
    \right)
    \right):=R \label{eq:def-R}
\end{align}
where $C_5=2(\nu/\alpha)C_3$ and $C_6=\max(C_2, 2(\nu/\alpha))$ and $C_7=2C_1+1$. In the last inequality the constant $3$ is negligible if $\ell\geq 3$ and we bounded $\tr(\tilde{M}) = \tr(\tilde{Q}K_\eta)$ from Equation (\ref{eq:bound-tildeQK_eta}). %
Now we rewrite the problem so as to quantify the generalization error via Rademacher complexity. 
We define the loss function
$$l(f_M,(x,S)) := -\log(\tr(M B_{x,S}) + \alpha_S) + \lambda \tr(M),$$
where $f_M(B)=tr(MB)$ and $B=B_{x,S}= K_\eta^{-1/2} A_{x,S} K_\eta^{-1/2},$ where $A_{x,S}$ is defined in (\ref{eq:A}). Observe that $\Lambda(M) = \frac{1}{N}\sum_{i=1}^N l (f_M,(x_i,S_i))$ for any $M.$ So Problem~\eqref{eq:primal} is equivalent to
$$\hat{f} = \min_{f \in {\cal F}_R} \frac{1}{N}\sum_{i=1}^N l (f,(x_i,S_i)),$$
where ${\cal F}_R = \{f_M ~|~ M \in \R^{\ell \times \ell}, M \succeq 0, \tr(M C) = 1, \tr(M) \leq R \}$, where $R$ as defined in Equation (\ref{eq:def-R}) is large enough that $\tr(\hat{M}) \leq R$, so $\hat{f} = f_{\hat{M}} \in {\cal F}_R$ and also $\tr(\tilde{M}) \leq R$, so $f_{\tilde{M}} \in {\cal F}_R$.

Now, given a dataset $D=\{(x_i,S_i)\}_{i=1,\dots,N}$ we need to bound the Rademacher complexity of $l\circ {\cal F}_R\circ D,$ so that we can bound the excess risk. Observe that 
\begin{align*}
    R(l\circ {\cal F}_R\circ D) 
    &=\frac{1}{N}\mathbb{E}_{\sigma }\left[
        \sup_{f_M\in \mathcal{F}_R}\sum_{i=1}^N\sigma_il(f,(x_{S_i}, S_i ))
    \right] \\
    &=\frac{1}{N}\mathbb{E}_{\sigma}\left[
        \sup_{f_M\in \mathcal{F}_R}-\sum_{i=1}^N\sigma_i
        \log( \tr(M B_{x_i,S_i}) + \alpha_{S_i}) + \sum_{i=1}^N\sigma_i\lambda \tr(M)
    \right] \\
    &\leq\frac{1}{N}\mathbb{E}_{\sigma}\left[
        \sup_{f_M\in \mathcal{F}_R}-\sum_{i=1}^N\sigma_i
        \log( \tr(M B_{x_i,S_i}) + \alpha_{S_i}) + 
        \sum_{i=1}^N\sigma_i\sup_{f_M\in \mathcal{F}_R}\lambda \tr(M)
    \right] \\
    &=\frac{1}{N}\mathbb{E}_{\sigma}\left[
        \sup_{f_M\in \mathcal{F}_R}-\sum_{i=1}^N\sigma_i
        \log( \tr(M B_{x_i,S_i}) + \alpha_{S_i}) 
    \right] \\
    &=R(\psi_D\circ {\cal F}_R \circ B_D)
\end{align*}
where $\sigma=(\sigma_i)_{i=1,\dots,N}$ is a vector of Rademacher random variables, then $\psi_D:\R^N\to\R^N$ is defined as $\psi_D(a_1,\dots,a_N)=(-\log(a_i+\alpha_{S_i}))_{i=1,\dots,N}$, and $D_B$ is defined as the image of the dataset through the feature map, $i.e.$ $B_D=(B_{x_i,S_i})_{i=1,\dots,N}.$ If we denote $\alpha$ as before, then $\psi_D$ is component--wise $1/\alpha$ Lipschitz. We are working in a
subset of the ball given by the trace norm so the Rademacher complexity can be bounded by 
\begin{align*}
    R(l\circ {\cal F}_R\circ D)\leq \frac{1}{\alpha}R({\cal F}_R\circ B_D)\leq\frac{R}{\alpha}\sqrt{\frac{\log (\ell)}{N}}
\end{align*}
where in the first inequality we are using Lemma 26.9 of \cite{shalev2014understanding} to bound the Lipschitz part and in the second inequality we are using \cite{kakade2012regularization}, Table 1, case $W_{S(1)}$ plus the bound from \cref{lem:AinvK} that the embedding of the training set $B_D$ has operator norm bounded by one.

We can now bound the generalization error (see for example Theorem 26.5 point 3 in Section 26.1 \cite{shalev2014understanding}), obtaining with probability $1-\tau$, 
\begin{align*}
\tilde{\cal L}(\hat{Q}) - \tilde{\cal L}(\tilde{Q}) &\leq \frac{2 R \sqrt{\log \ell}}{\alpha \sqrt{N}} + 5 c \sqrt{\frac{2 \log(8/\tau)}{N}} \\
& \leq \frac{2 R \sqrt{\log \ell}}{\alpha \sqrt{N}} + 5\left(\log\frac{R}{\alpha} +\lambda R \right)  \sqrt{\frac{2 \log(8/\tau)}{N}} \\
& \leq \frac{C}{\nu} \left(\ell + \lambda + \frac{1}{\lambda}\right) \log^{3/2}\left(\frac{\ell\log(1/\lambda\nu)}{\lambda\nu}\right) \sqrt{\frac{\log(8/\tau)}{N}},
\end{align*}
where 
$c = \sup_{x\in D,S \subseteq \{1,\dots, d\}, f_M \in {\cal F}_R} |l(f_M,(x,S))| \leq \log(R/\alpha) + \lambda R$ where we use \cref{lem:AinvK} to bound $\tr(K_\eta^{-1/2} A_{x,S} K_\eta^{-1/2})$ for any $x, S$. 
The proof is concluded by taking the intersection bound of the two events, the first with probability $1-\tau$ that we required at the beginning of the proof and the second being the generalization bound above.

\end{proof}

\subsection{Details of Thm.~\ref{thm:marginal}}\label{appx:detail_th_bound_the_marginal}
From Theorem \ref{thm:approx-error-tildeQ} and Theorem \ref{thm:learn}, we get that
\begin{align*}
    \int \! \mathrm{KL}(p^\star_S\,\|\,\hat{p}_S)\,\mathrm d\mu(S) 
&\leq \left(\tilde{\cal L}(\hat{Q}) - \tilde{\cal L}(\tilde{Q})\right) + \tilde{\cal L}(\tilde{Q})\\
&\leq \frac{C}{\nu} \left(\ell + \lambda + \frac{1}{\lambda}\right) \log^{3/2}\left(\frac{\ell\log(1/\lambda\nu)}{\lambda\nu}\right) \sqrt{\frac{\log(8/\tau)}{N}}\\
&+C_1\log(C_2/\nu)\nu + C_3\ell\lambda.
\end{align*}
We now study a simpler function and we will get a bound that asymptotically holds also for the above, weakened just by logarithmic terms.
Define
\begin{align*}
    f(\lambda,\nu)=\frac{\ell}{\nu\sqrt{N}} + \frac{\lambda}{\nu\sqrt{N}} + \frac{1}{\nu\lambda\sqrt{N}}+\nu+\ell\lambda.
\end{align*}
We expect $\lambda$ to be small, so we study
\begin{align*}
    g(\lambda,\nu)=\frac{\ell}{\nu\sqrt{N}} + \frac{1}{\nu\lambda\sqrt{N}}+\nu+\ell\lambda.
\end{align*}
We can then bound first in $\lambda,$ and we will get a bound in $\nu.$ We then replace from Theorem (\ref{thm:exists-barQ}) the quantity $\ell=\nu^{-d/s},$ and we optimize in $\nu,$ and we get the quantity from the theorem (we have set $\varepsilon=\nu$ in the theorem). The other quantities follow.

\section{Extra Lemmas}
\label{appx:extra-lemmas}
In this part of the appendix we prove some useful lemmas to obtain the main results of the paper.
\label{appx:lemmas}
\begin{lemma}\label{lem:AinvK}
For any mask $S \subset \{1,\dots,d\}$ and any $x \in D$ define the matrix $A_{x,S}$ as
$$
A_{x,S}  = \left(\circ_{j \in S} \,\phi_j(x_j) \phi_j(x_j)^\top\right) \circ \left(\circ_{j \in \bar{S}} \,H_j\right),
$$
where $\circ$ is the Hadamard product, $\phi_j(z) = (k_\eta(z,w_{1,j}), \dots k_\eta(z,w_{\ell,j}))$ and $w_{i,j}$ is the $j$-th component of the mesh point $w_i \in \R^d$ and $H_j = \frac{1}{V_D}\int_D \phi_j(z) \phi_j(z)^\top dz$. Then
$$\|K_\eta^{-1/2}A_{x,S} K_\eta^{-1/2}\|_{op} \leq 1.$$
\end{lemma}
\begin{proof}
Let ${\cal G}_\eta$ be the reproducing Kernel Hilbert space associated to the Gaussian kernel of width $\eta$ on $\R$. Denote by $k_z \in {\cal G}_\eta$ the reproducing vector associated to $z \in \R$, i.e. $\scal{k_z}{k_{z'}}_{{\cal G}_\eta} = k_\eta(z,z')$. Let $Z_j:{\cal G}_\eta \to \R^\ell$ be the finite rank linear operator defined as $Z_j f = (f(w_{1,j}), \dots, f(w_{\ell,j}))$ for any $f \in {\cal G}_\eta$. Define the kernel matrices $K_j \in \R^{\ell \times \ell}$ as $(K_j)_{p,q} = k_\eta(w_{p,j},w_{q,j})$ for $p,q \in \{1,\dots, \ell\}$. Note that 
$$K_j = Z_j Z_j^*, ~~\phi_j(z) = Z_j k_z,~~H_j = \frac{1}{V_D}\int_D (Z_j k_z)(k_z^* Z_j^*) du(z) = Z_j C Z_j^*,$$
where $C = \frac{1}{V_D}\int_D k_z k_z^* dz:{\cal G}_\eta \to {\cal G}_\eta$ with $\tr(C) = \frac{1}{V_D} \int \|k_z\|^2 dz = 1$. Now note that for any $j \in \{1,\dots,d\}$
$$H_j = Z_j C Z_j^* \preceq \|C\|_{op}Z_j Z_j^*  = K_j,$$
moreover
$$\phi_j(z) \phi_j(z)^\top = Z_j k_z k_z^* Z_j^* ~\preceq~ \|k_z\|^2_{{\cal G}_\eta} Z_j Z_j^* = K_j.$$
Since $A \circ B \preceq C \circ D$ when $A \preceq C$ and $B \preceq D$, we have that
$$A_{x,S} ~\preceq \circ_{j \in \{1,\dots,d\}} Z_j Z_j^* = K_\eta.$$
This implies that $K_\eta^{-1/2} A_{x,S} K_\eta^{-1/2} \preceq I$, so $\|A_{x,S}^{1/2} K_\eta^{-1/2}\|_{op} \leq 1$ and also
$$\tr(A_{x,S} K_\eta^{-1}) = \tr(K_\eta^{-1/2} A_{x,S} K_\eta^{-1/2}) \leq \tr(I) = \ell.$$
\end{proof}

\begin{lemma}\label{lemma:W_Lambert_bound}
    Let $x,y,\alpha,\lambda\in\R^{++}$ be strictly positive numbers. If $x\leq y+(1/\lambda)\log(1+x/\alpha),$ then
    \begin{align*}
    x
    &\leq
    \frac{2}{\lambda}\log(1/(\alpha\lambda)) + 2y+2\alpha
\end{align*}

\end{lemma}
\begin{proof}
We start from the inequality $x\leq y+(1/\lambda)\log(1+x/\alpha).$ Then
\begin{align*}
    \tilde x\leq \frac{y}{\alpha}+1+\frac{1}{\alpha\lambda}\log\tilde x,\quad \tilde x=\frac{x}{\alpha}+1.
\end{align*}
Then
\begin{align*}
    &\alpha\lambda\tilde x\leq \lambda y+\alpha\lambda+\log\tilde x\\
    &e^{\alpha\lambda\tilde x}\leq e^{\lambda y+\lambda\alpha}\tilde x\\
    &e^{-\lambda y-\lambda\alpha}\leq \tilde xe^{-\alpha\lambda\tilde x}\\
    &-\alpha\lambda e^{-\lambda y-\lambda\alpha}\geq (-\alpha\lambda\tilde x)e^{-\alpha\lambda\tilde x}\\
    &-\alpha\lambda e^{-\lambda y-\lambda\alpha}\geq (-\lambda(x+\alpha))e^{-\lambda(x+\alpha)}
\end{align*}
Denote $\varepsilon=\alpha\lambda e^{-\lambda y-\lambda\alpha}.$ Now, an equation of the form $0>-\varepsilon\geq ye^y\geq -1/e$ has solutions $W_0(-\varepsilon)\geq y\geq W_{-1}(-\varepsilon),$ where $W_0$ and $W_{-1}$ are the two real branches of the Lambert $W-$function. In our case we are interested in an upper bound, so we need to bound $-W_{-1}(-\varepsilon).$ From \cite{loczi2020explicit}, for $t\in[-1/e,0),$ we have $W_{-1}(t)\geq e\log(-t)/(e-1)$ so, since $ye^{y}\in[-1/e,0)$ for $y=-\lambda(x+\alpha),$ we get
\begin{align*}
    \alpha\lambda+\lambda x\leq -W_{-1}(-\varepsilon)\leq -\frac{e}{e-1}\log(\varepsilon)=
    \frac{e}{e-1}
    \left(
    \log(1/\alpha) + \log(1/\lambda) +\lambda\alpha +\lambda y
    \right),
\end{align*}
that is, observing that $\lambda x\leq \lambda x+\alpha\lambda$ and dividing by $\lambda\alpha$ we get
\begin{align*}
    \frac{x}{\alpha}+1
    &\leq \frac{e}{e-1}
    \left(
    \frac{\log(1/\alpha)}{\alpha\lambda} + \frac{\log(1/\lambda)}{\alpha\lambda} + \frac{y}{\alpha}+1
    \right)\\
    &\leq
    2\frac{\log(1/(\alpha\lambda))}{\alpha\lambda} + 2\frac{y}{\alpha}+2.
\end{align*}
Finally, when we simplify the one on the left hand side and we multiply by $\alpha$ we get
\begin{align*}
    x
    &\leq
    \frac{2}{\lambda}\log(1/(\alpha\lambda)) + 2y+2\alpha
\end{align*}
as we wanted.
\end{proof}

\begin{corollary}\label{corollary:W_Lambert_bound_corollary}
    Let $x,y,\alpha,\lambda\in\R^{++}$ be strictly positive numbers. If $x\leq y+(1/\lambda)\log(1+x/\alpha),$ then
\begin{align*}
    x
    &\leq y +\frac{1}{\lambda}
    \log\left(3+  2\frac{\log\left(1/(\alpha\lambda)\right)}{\alpha\lambda} + 2\frac{y}{\alpha}
    \right).
\end{align*}
\end{corollary}
\begin{proof}
    It follows by plugging the bound of Lemma \ref{lemma:W_Lambert_bound} into the logarithm.
\end{proof}

\section{Computations and Implementation}
\label{appx:comput}

\subsection{Overview of the section}

In this section we summarize the best strategy to solve our optimization problem, based on the number of nodes $r$ and the size of the dataset $N$. 
In general, explicitly forming the Hessian would require $O((\ell^2)^2)=O(\ell^4)$ memory and $O(\ell^6)$ time for the resolution of the linear system, which can become infeasible in practice despite being theoretically motivated.
Moreover, the computation requires the inversion of certain matrices and may be numerically unstable. To address these issues, in Section \ref{appx:comp_grad_hess} we compute the gradient and the Hessian of the loss function, and derive simplifications for solving the Newton system. 
We then present in Section \ref{sct:direct_method} a direct method based on the Sherman–Morrison–Woodbury formula, enabling efficient computation of Hessian–vector products. In Section \ref{appx:cg-good} we introduce an iterative procedure based on the conjugate gradient method, where we investigate the properties of a preconditioner to mitigate numerical issues. Finally, in Section \ref{appx:hyper_parameters_optimization} we present a practical algorithm to solve the non-convex problem.

\subsection{Details on PSD--Impute and PSD--Convex} 

The following pseudo--code shows how to solve \textbf{PSD--Impute}. The step (h) is optional and in case it is skipped, we get \textbf{PSD--Convex}. In practice, steps (c)–(g) may be repeated until convergence or stopped earlier; the same applies to step (h).

\label{eq:interior-damped-Newton}
\begin{center}
\fbox{\begin{minipage}{0.95\linewidth}
\textbf{PSD-Impute: Algorithm for the solution of Problem ~\eqref{eq:primal-short}}\label{alg:PSD_Impute}\\[2pt]
\emph{Input:} 
regularization parameter $\lambda,$ log--det parameter $\mu$, tolerance $\varepsilon$. \\
Starting anchor nodes $\mathcal{W}_0$ and precision $\eta_0.$ A starting positive definite matrix $Q_0.$
\begin{enumerate}[leftmargin=*]
  \item Initialise $Q=Q_0,\mathcal{W}=\mathcal{W}_0,\eta=\eta_0.$ 
  \item \textbf{for} $t=0,1,\dots$ \textbf{do}
    \begin{enumerate}[leftmargin=*]
      \item Compute matrices $A_i$, $i=0,\dots,N$ and $H$ from the current $\mathcal{W}$ and $\eta$.
      \item renormalize $Q$ so that $\Tr(QH)=1$.
      \item Compute gradient %
      and Hessian--vector products. %
      \item Solve for Newton direction $\Delta$ with SMW.
      \item Newton decrement $\lambda^2=-\langle  g(Q),\Delta\rangle$. \textbf{if} $\lambda^2/2\le\varepsilon$ \textbf{stop}.
      \item Backtracking line search: start $\alpha=1$, reduce $\alpha\leftarrow\beta\alpha$ until $f(Q+\alpha\Delta,\mathcal{W},\eta)$ satisfies the Armijo rule. We suppose $f$ attains $+\infty$ if $Q+\alpha\Delta$ is outside the PSD cone.
      \item Update $Q\leftarrow Q+\alpha\Delta$ and rescale so that $\Tr(QH)=1$, for numerical stability.\\
      \textbf{(Optional) (anchor, precision) update:}
      \item Move the pair $(\mathcal W,\eta),$ i.e. the anchors $\mathcal W$ and the kernel precision $\eta$, towards locations of highest density via some gradient step, following the procedure in Section \ref{appx:hyper_parameters_optimization}.\\
    \end{enumerate}
\end{enumerate}
\emph{Output:} optimum $\hat Q$ and density $p_{\hat Q}$.\end{minipage}}
\end{center}

\subsection{Computation of the gradient and Hessian}\label{appx:comp_grad_hess}
In this section we derive the expression for the gradient and the Hessian of the loss function, namely Problem (\ref{eq:primal-short}). These computations are classical, see for example \cite[Appendix A]{boyd2004convex}.
We recall that Problem (\ref{eq:primal-short}) is defined by the loss function
\begin{align*}
  f_{\alpha}(Q):=-\frac1N\sum_{i=1}^N\log\Tr(QA_i+\alpha)+\lambda\Tr(QA_0)-\mu\log\det Q.
\end{align*}
where $Q$ is a positive semidefinite matrix, such that $Tr(QH)=1,$ with $H$ a problem dependent positive definite matrix. 
Denote by $S_i(Q,\alpha)=-\log(Tr(QA_i)+\alpha)$. Then a Taylor expansion shows that the gradient and the Hessian of $S_i$ can be written as
\begin{align*}
    \nabla_QS_i(Q,\alpha)
    &=-\frac{A_i}{Tr(QA_i)+\alpha},\quad\nabla^2_QS_i(Q,\alpha)
    =\frac{A_i\otimes A_i}{(Tr(QA_i)+\alpha)^2},
\end{align*}
where $(A_i\otimes A_i)(E)=A_iTr(EA_i)$ is a rank one tensor in the space of symmetric matrices.

So the final expression for the gradient and the Hessian of $f_{\alpha}$ will be
\begin{align}\label{eq:grad_hess_alpha}
    g_{\alpha,Q}&=-\frac{1}{N}\sum_{i=1}^N\frac{A_i}{Tr(QA_i)+\alpha} + \lambda A_0-\mu Q^{-1},\\
    G_{\alpha,Q}&=\frac{1}{N}\sum_{i=1}^N\frac{A_i\otimes A_i}{(Tr(QA_i)+\alpha)^2} +\mu Q^{-1}\otimes_K Q^{-1}. \nonumber
\end{align}
Observe that
\begin{align}\label{eq:hess_alpha_times_Q}
    G_{\alpha,Q}(Q)
    &=\frac{1}{N}\sum_{i=1}^N\frac{Tr(QA_i)A_i}{(Tr(QA_i)+\alpha)^2} + \mu Q^{-1}\\
    &=\frac{1}{N}\sum_{i=1}^N\frac{A_i}{Tr(QA_i)+\alpha} + \mu Q^{-1} 
    -\alpha\tilde A(\alpha),\quad 
    \tilde A(\alpha)=\frac{1}{N}\sum_{i=1}^N\frac{A_i}{(Tr(QA_i)+\alpha)^2} \nonumber\\
    &=-g_{Q,\alpha}+\lambda A_0 -\alpha\tilde A(\alpha)= -g_{Q,\alpha}-B_{\lambda,\alpha},\quad B_{\lambda,\alpha}=-\lambda A_0 +\alpha\tilde A(\alpha) \nonumber
\end{align}
so we have an expression for the Hessian-gradient product.
\begin{align}\label{eq:inv_hess_grad}
    G^{-1}_{\alpha,Q}(-g_{Q,\alpha})=Q-\lambda G^{-1}_{\alpha,Q}(A_0)+\alpha G^{-1}_{\alpha,Q}(\tilde A(\alpha))
\end{align}

We can write the Newton step as
\begin{align}\label{eq:newton_step_and_decrement_our_problem_up}
    \Delta Q&=G^{-1}_Q(-g_Q)-\nu_QG_Q^{-1}(H),\\
    \nu_Q&=-\frac{\langle G^{-1}_Q(H),-g_Q\rangle_{\mathcal{S}_d}}{-\langle G^{-1}_Q(H),H\rangle_{\mathcal{S}_d}}, \nonumber
\end{align}
and the Newton decrement will be $\lambda_Q^2=-Tr(g_Q\Delta Q).$
Now we can efficiently write the Newton step and the Newton decrement.

\begin{theorem}\label{th:nt_step_nt_decr}
Define $B_{\lambda,\alpha}=-\lambda A_0+\alpha\tilde A(\alpha).$
Then the Newton step is defined by the equations \begin{align}\label{eq:first_eq_th_NS}
&\Delta Q = Q+G_{\alpha,Q}^{-1}(B_{\lambda,\alpha})- \nu(Q)G^{-1}_{\alpha,Q}(H),\\
&\nu(Q)=
-\frac{1+\langle G^{-1}_Q(H), B_{\lambda,\alpha}\rangle_{\mathcal{S}_d} }{-\langle G^{-1}_Q(H), H\rangle_{\mathcal{S}_d}}. \nonumber
\end{align}
Moreover, the Newton decrement can be written as 
\begin{align*}
    \lambda(Q)^2&=c_{\alpha,Q}+\ell\mu + Tr(QB_{\lambda,\alpha})-\nu_Q+ Tr(B_{\lambda,\alpha}\Delta Q),
\end{align*}
where $\ell$ is the size of the matrix $A_0,$ and $c_{\alpha,Q}\in(0,1]$ is bounded.
\end{theorem}
\begin{proof}
Equation (\ref{eq:first_eq_th_NS}) follows by plugging (\ref{eq:inv_hess_grad}) in (\ref{eq:newton_step_and_decrement_our_problem_up}).
For the second equation observe that
\begin{align*}
    \nu_Q&=-\frac{\langle H,G^{-1}_Q(-g_Q)\rangle_{\mathcal{S}_d}}{-\langle G^{-1}_Q(H),H\rangle_{\mathcal{S}_d}}
    =-\frac{\langle H, Q+G^{-1}_Q(B_{\lambda,\alpha}) \rangle_{\mathcal{S}_d}}{-\langle G^{-1}_Q(H) ,H\rangle_{\mathcal{S}_d}}
    =-\frac{1+\langle G_Q^{-1}(H), B_{\lambda,\alpha} \rangle_{\mathcal{S}_d}}{-\langle G^{-1}_Q(H) ,H\rangle_{\mathcal{S}_d}},
\end{align*}
as we wanted.
For the Newton decrement, we have
\begin{align*}
    \lambda_Q^2&=\langle-g_Q,\Delta Q\rangle\\
    &=\langle G_Q(Q) + B_{\lambda,\alpha}  , \Delta Q   \rangle  \\
    &=\langle G_Q(Q),Q+G_Q^{-1}(B_{\lambda,\alpha} -\nu_QH)\rangle  +\langle B_{\lambda,\alpha}  , \Delta Q   \rangle\\
    &=\langle G_Q(Q),Q\rangle
    +\langle Q,B_{\lambda,\alpha}-\nu_QH)\rangle  +\langle B_{\lambda,\alpha} , \Delta Q   \rangle,
\end{align*}

From (\ref{eq:hess_alpha_times_Q}) we immediately get that 
\begin{align}\label{eq:Hess_times_Q_Q_alpha}
    \langle G_Q(Q),Q\rangle=\frac{1}{N}\sum_{i=1}^N\frac{Tr(QA_i)^2}{(Tr(QA_i)+\alpha)^2} +\mu \ell =c_{\alpha,Q}+\mu\ell 
\end{align}
so the final expression for $\lambda^2_Q$ is
\begin{align*}
     \lambda_Q^2
    &=\langle G_Q(Q),Q\rangle
    +\langle Q,B_{\lambda,\alpha} -\nu_QH)\rangle  +\langle B_{\lambda,\alpha}  , \Delta Q   \rangle\\
    &=c_{\alpha,Q}+\ell\mu + Tr(QB_{\lambda,\alpha})-\nu_Q+ Tr(B_{\lambda,\alpha}\Delta Q),
\end{align*}
where $c_{\alpha}$ is defined in (\ref{eq:Hess_times_Q_Q_alpha}), as we wanted.

\end{proof}

\subsection{Direct method}\label{sct:direct_method}
To be able to perform the computation of $\Delta Q$ and $\lambda^2_Q,$ we need to be able to invert the Hessian matrix. To do that, we can leverage the SMW formula. Define
\begin{align*}
    A_{Q,\alpha}&: \mathcal{S}_p\to \R^N,\quad 
    A_{Q,\alpha}(W) = \frac{1}{\sqrt{ N}}\left(
    \frac{Tr(A_1W)}{Tr(A_1Q)+\alpha},..,\frac{Tr(A_NW)}{Tr(A_NQ)+\alpha}
    \right)^T\\
    A_{Q,\alpha}^*&:\R^N\to\mathcal{S}_p,\quad A^*_{Q,\alpha}(x)=\frac{1}{\sqrt N}\sum_{i=1}^Nx_i\frac{A_i}{Tr(A_iQ)+\alpha}
\end{align*}
Then $G_Q= A_{Q,\alpha}^*A_{Q,\alpha} + \mu Q^{-1}\otimes_K Q^{-1},$ and the explicit expression for $G^{-1}_{Q,\alpha}$ is
\begin{align*}
    \mu G_{Q,\alpha}^{-1} = Q\otimes_K Q - (Q\otimes_K Q)A^*(\mu Id_N + A(Q\otimes_K Q)A^*)^{-1}A(Q\otimes_K Q),\quad A=A_{Q,\alpha}
\end{align*}
This expression provides a direct algorithm for the computation of the Newton step, following Theorem \ref{th:nt_step_nt_decr}.

The cost of the multiplication of the inverse of the Hessian with a matrix is dominated by forming the matrix $A (Q\otimes_K Q) A^*$, that is $O(N\ell^3 + N^2\ell^2),$ that is essentially the cost of building $QA_i$ for $i=1,\dots,N$ and then computing all the possible $Tr(QA_iQA_j).$
Then comes the cost of inverting $\mu Id + A(Q\otimes_K Q) A^*,$ that is $O(N^3).$ The number of iterations needed to produce a solution within an accuracy $\varepsilon$ is in practice proportional to the gap between the value of the loss function at the starting point and the optimal value, see \cite[Section 10.2]{boyd2004convex}.
In the end, the cost of running the Newton method and computing the Newton step is $O((L(Q_0)-L(Q^*))\cdot(N\ell^3 + N^2\ell^2 + N^3)),$ where $Q_0$ is the starting point and $Q^*$ is the minimizer of the functional. The cost in memory is dominated by storing the $N$ matrices $A_i,$ and the matrix $\mu Id_N+A(Q\otimes_K Q)A^*$, that is $N\times N.$ The total memory complexity is thus dominated by $O(N^2+N\ell^2).$

\subsection{Preconditioning the Hessian system}
\label{appx:cg-good}
In this section we investigate the preconditioning property of $Q\otimes_K Q.$ We first recall the basic notions of preconditioning. Let $Ax=b$ be a linear system, with $x$ the unknown and $A$ a positive symmetric matrix.
We recall, from \cite[Chapter 5, p. 118]{wright2006numerical}, that preconditioning is a change of variable from $x$ to $\hat x$ via a nonsingular matrix $C,$ that is $\hat x=Cx,$ and the system to be solved will be $(C^{-T}AC^{-1})\hat{x}=C^{-T}b.$ 
The idea is to find a matrix $C$ such that the preconditioned matrix $C^{-T}AC^{-1}\sim Id.$
One way to assess the stability of a linear system and make it amenable to powerful iterative methods that scale to large datasets, like Conjugate Gradient, is through the condition number \cite[p.~117]{wright2006numerical}, defined as the ratio between the largest and smallest eigenvalue of a matrix $A$, that is $k(A)=\lambda_{max}(A)/\lambda_{min}(A).$
Recall that the Hessian of the loss function is
\begin{align*}
    \nabla^2 g(Q)=\frac{1}{N}\sum_{i=1}^N\frac{A_i\otimes A_i}{(Tr(QA_i)+\alpha)^2} + \mu Q^{-1}\otimes_K Q^{-1}
\end{align*}
Here we see $A_i\otimes A_i$ as a rank one operator, that is $(A_i\otimes A_i)(H)=Tr(A_iH)A_i.$
Consider the change of variable $C_Q=(Q^{-1/2}\otimes_{K} Q^{-1/2}).$
Then 
\begin{align}\label{eq:change_of_variable_Hess}
    C^{-T}_Q\nabla^2 g(Q)C_Q^{-1} 
    &= \frac{1}{N}\sum_{i=1}^N\frac{(Q^{1/2}A_iQ^{1/2})\otimes (Q^{1/2}A_iQ^{1/2})}{(Tr(QA_i)+\alpha)^2} + \mu Id\\
    &=L_{Q,\alpha} + \mu Id, \nonumber
\end{align}
where $L_Q$ is the sum term in (\ref{eq:change_of_variable_Hess}).
We need to bound the biggest eigenvalue of $L_Q$; to do that we use the variational characterization of eigenvalues, see for example \cite[Section 3.1]{bhatia2013matrix}. Let $H\in\mathcal{S}_d$ be a symmetric matrix such that $||H||_{\mathcal{S}_d}=Tr(H^2)^{1/2}=1$. Denote by $B_{i,Q}=Q^{1/2}A_iQ^{1/2}$.
Then 
\begin{align*}
    \langle L_Q(H),H\rangle_{\mathcal{S}_d}=
    \frac{1}{n}\sum_{i=1}^n
    \frac{Tr(B_{i,Q}H)^2}{(Tr(QA_i)+\alpha)^2}
    \leq 
    \frac{1}{n}\sum_{i=1}^n\frac{Tr(B_{i,Q}^2)}{(Tr(B_{i,Q})+\alpha)^2}\leq 1,
\end{align*}
where we have used in the last inequality that for every positive semidefinite matrix $B,$ it holds $Tr(B^2)\leq Tr(B)^2.$ Then it follows that the condition number of the preconditioned matrix (\ref{eq:change_of_variable_Hess}) is
\begin{align}\label{eq:bound_cond_number}
    k(L_Q+\mu Id)=
    \frac{\lambda_{max}(L_Q+\mu Id)}{\lambda_{min}(L_Q+\mu Id)}\leq\frac{1+\mu}{\lambda_{min}(L_Q)+\mu }\leq \frac{1+\mu}{\mu}\leq 1+\frac{1}{\mu}.
\end{align}
Observe that in practice, we do not need to compute the matrix $Q^{1/2}$, because the preconditioning algorithm works with the full matrix $Q$, see \cite[p.~118]{wright2006numerical}.
The conjugate gradient method is known to reach an $\varepsilon$-accurate solution of the linear system in a number of iterations $K$ proportional to the square root of the condition number, see \cite[p.~117]{wright2006numerical}. So, from (\ref{eq:bound_cond_number}) we obtain that the number of steps $K$ needed is bounded by $K=O(\log(1/\varepsilon)\sqrt{1+1/\mu}).$ Since the cost of a Hessian-matrix product is on the order of $O(\ell^3+\ell^2N)$ for the change of variable and the computation of the sum of the tensors, the total cost to obtain an $\varepsilon$--solution will be $O((\ell^3+\ell^2N)\log(1/\varepsilon)\sqrt{1+1/\mu}).$

\subsection{Hyper-parameter optimization}
\label{appx:hyper_parameters_optimization}
What we have presented assumes the anchor points $X$ and the precision of the Gaussian $\eta$ to be known a priori, but this is in general not possible. This is a common problem for example in the context of Gaussian processes \cite{NIPS2005_4491777b} and mixture models \cite{bilmes1998gentle}. To overcome this issue, we extend our problem, and we consider
\begin{align}\tag{$P_E$}\label{argmin:extended_argmin}
\argmin_{\substack{
    Q \succcurlyeq 0, 
    \\X \in \mathbb{R}^{r \times d}, 
    \eta \in \mathbb{R}^d_{+} \\
    \mathrm{Tr}(QH(X, \eta)) = 1
}}
L(Q,X,\eta)=&-\frac{1}{N}\sum_{i=1}^N\log Tr(Q\cdot A_i(X,\eta)+\alpha) \\
& +\lambda Tr(Q\cdot A_0(X,\eta))
-\mu\log |Q| \nonumber,
\end{align}
where we have made explicit the dependence of the matrices on the nodes $X$ and the precision $\eta.$
We need to find a way to remove the trace constraint. The key observation is the following. Let $(X,\eta)$ be fixed hyperparameters. We can find a matrix $Q^*=Q(X,\eta)$ that is the solution of $L(\cdot,X,\eta).$ Since this point will be in the interior of the positive definite cone, $KKT$ conditions must hold \cite[Section 5.5.3]{boyd2004convex}, so $Tr(Q^*H(X,\eta))=1,$ and
\begin{align*}
    \nabla_QL(Q^*,X,\eta) + \omega^*H(X,\eta)=0.
\end{align*}
Since the problem is convex and the constraint is linear, $Q^*$ and $\omega^*$ are primal and dual optimal, with zero duality gap. 
The dual variable $\omega^*$ can be computed in closed form. Indeed, from (\ref{eq:grad_hess_alpha}) we deduce that $\langle g_Q,Q\rangle=-b_{\alpha,Q}+\lambda Tr(A_0Q)-\mu \ell,$ with $b_{\alpha,Q}=(1/N)\sum_{i=1}^NTr(QA_i)/(Tr(QA_i)+\alpha)$ and computing the scalar product with $Q^*$ and rearranging the terms we get
\begin{align*}
    -b_{\alpha,Q^*} + \lambda Tr(Q^*A_0) -\mu\ell + \omega^*=0 \iff 
    \omega^*=\omega(X,\eta)=b_{\alpha,Q(X,\eta)}+\mu \ell -\lambda Tr(Q(X,\eta)A_0).
\end{align*}
Now, define the problem 
\begin{align}\tag{$\tilde P$}\label{argmin:unconstrained_problem}
    \argmin_{\substack{
    Q \succcurlyeq 0, \\X \in \mathbb{R}^{r \times d}, 
    \eta \in \mathbb{R}^d_{+}
}}\tilde L(Q,X,\eta)=L(Q,X,\eta) + \omega(X,\eta)(Tr(QH_{X,\eta})-1).
\end{align}
The function $\tilde L$ for $(X,\eta)$ fixed is nothing more than the Lagrangian of the original problem. Observe that $Q(X,\eta)$ minimizes $\tilde L(\cdot,X,\eta).$
It turns out that minimizing $L$ and $\tilde L$ is the same.%
\begin{theorem}\label{th:equality_constr_and_unconstr_problem}
We have that problem (\ref{argmin:extended_argmin}) and problem (\ref{argmin:unconstrained_problem}) have the same infimum value, that is 
\begin{align*}    
    \inf_{
    \substack{
    Q\succcurlyeq 0,X,\eta \\
    Tr(QH_{X,\eta}) = 1}} L(Q,X,\eta) = 
    \inf_{Q\succcurlyeq 0,X,\eta}
    \tilde{L}(Q,X,\eta)
\end{align*}
\end{theorem}
\begin{proof}
Denote by $(L)$ the infimum on the LHS and $(R)$ the infimum on the RHS. If $Q,X,\eta$ is a triple as in the LHS, then   $L(Q,X,\eta)=\tilde L(Q,X,\eta)\geq (R).$ For the opposite direction, fix $Q,X,\eta$ as in the right--hand side. Then 
\begin{align*}
    \tilde L(Q,X,\eta)\geq \tilde L(Q(X,\eta),X,\eta)=L(Q(X,\eta),X,\eta)\geq (L),
\end{align*}
where we have used in the first inequality that $Q(X,\eta)$ minimizes the Lagrangian of the system, and in the equality we have used that $Tr(Q(X,\eta)H_{X,\eta})=1.$
\end{proof}
To avoid pathological situations where the infima are not minima, we can for example set $A_0(X,\eta)$ as a constant full rank matrix, so that the $\log$--determinant cannot escape to infinity. 
Now, Theorem (\ref{th:equality_constr_and_unconstr_problem}) allows us to perform alternating minimization \cite{byrne2013alternating}. If $Q$ is feasible for (\ref{argmin:extended_argmin}), the derivative of $\tilde L$ with respect to $(X,\eta)$ is
\begin{align*}
    \nabla_{X,\eta}\tilde{L}(Q,X,\eta)
    &=\nabla_{X,\eta}L(Q,X,\eta) + \nabla_{X,\eta}\omega(X,\eta)(Tr(QH_{X,\eta})-1)) + \omega(X,\eta)\nabla_{X,\eta}Tr(QH_{\eta}) \\
    &=\nabla_{X,\eta}L(Q,X,\eta) + \omega(X,\eta)\nabla_{X,\eta}Tr(QH_{\eta}),
\end{align*}
where we are supposing that $\omega(X,\eta)$ depends smoothly on the hyperparameters. This justifies the heuristic that in a neighborhood of the feasible point, we can compute the derivative with respect to $(X,\eta)$ supposing the dual variable $\omega$ fixed. 

\subsection{Computation of the Conditional Mean of a PSD model}\label{app:psd_conditional_mean_imputation_psd}
We rely on the computation in Appendix F of \cite{rudi2021psd}. Let $p$ be a psd model with anchor points $\{w_1,..,w_m\}\subseteq\R^d.$
The conditional mean is defined as
\begin{align*}
    \mathbb{E}[X_{NS}|X_{S}=x_S]
    &=\int_{\R^{|NS|}}x_{NS}p(x_{NS}|x_S)dx_{NS}\\
    &=\frac{1}{p(x_S)}\int_{\R^{|NS|}}x_{NS}p(x_S,x_{NS})dx_{NS}
\end{align*}
where the integral is meant to be component-wise. Let $h$ be an index $h\in NS.$ Define $c_{h,\eta}(z)=\int_{\R^d} x_ie^{-\eta||w-z||^2}dw$. Then observe that the components different from $j$ are integrated out, and a factor $\sqrt{\pi/\eta_i}$ appears. For the $h$--th component we get
\begin{align*}
    \int w_he^{-\eta_i(w_h-z_h)^2}dw_h=\int_{\R}(w_h-z_h)e^{-\eta_h(w_h-z_h)^2}dw_h + \int_{\R}z_he^{-\eta_i(w_h-z_h)^2}dw_h=z_h\sqrt{\pi/\eta_h},
\end{align*}
so $c_{h,\eta}(z)=z_hc_{\eta_{NS}},$ with $c_{\eta_{NS}}=\pi^{|NS|/2}|\eta_{NS}|^{-1/2}.$ For the numerator, following \cite{rudi2021psd}, Appendix F, we first evaluate the PSD model at $x_S,$ and then we need to compute the mean with respect to $x_{NS}$. Denote by $v_{S}=(k_S(x_{1,S},w_S),..,k_S(x_{n,S},w_S)).$ Then evaluating at $x_{S}$ we still get a psd model with base matrix $Q\circ v_S^Tv_S,$ and when we compute the mean we get
\begin{align*}
    \mathbb{E}_{x\sim p}[X_h]=
    Tr((Q\circ v_Sv_S^T\circ K_{\eta_{NS}/2})G), \quad G_{i,j}=c_{h,2\eta_{NS}}\left(\frac{w_i+w_j}{2}\right)_h= \frac{(w_{i,h}+w_{j,h})}{2} c_{2\eta_{NS}},
\end{align*}
where $w_{i,h}$ is the $h$--th component of the $i$--th anchor vector.
For the denominator, we need to integrate out the component in $x_{NS},$ and we get
\begin{align*}
    p(x_S)=c_{2\eta_{NS}}Tr((Q\circ v_Sv_S^T)K_{\eta/2}).
\end{align*}
Now, define $\tilde G$ such that $G=\tilde Gc_{2\eta_{NS}}.$ In the end, we get that
\begin{align*}
    \mathbb{E}[X_{NS}|X_{S}=x_S]
    =
    \frac{Tr((Q\circ v_S^Tv_S\circ K_{\eta_{NS}/2})\tilde G)}{Tr((Q\circ v_Sv_S^T) K_{\eta_{NS}/2})}.
\end{align*}
Given the kernel matrices, the computation of the conditional mean is very efficient since it is just a computation of traces, and can be performed in $O(r^2),$ where $r$ is the number of anchor points. So, the total run time for imputing a dataset is $O(r^2n_{miss}),$ where $n_{miss}$ is the total number of missing components.

\end{document}